\documentclass[10pt]{article}

\usepackage[usenames,dvipsnames,svgnames,table]{xcolor}
\usepackage[colorlinks=true,
            linkcolor=black,
            urlcolor=blue,
            citecolor=blue]{hyperref}

\usepackage[square,numbers]{natbib}
\usepackage[normalem]{ulem}

\usepackage{enumitem}
\usepackage{sectsty,amsmath, amsfonts, amsthm, amssymb, mathtools,dsfont,mathrsfs}
\usepackage{tocvsec2}
\usepackage{fullpage}
\usepackage{subcaption}
\usepackage{float}	
\usepackage{forest, siunitx, booktabs}
\usepackage{bbm}
\usepackage{etoolbox}
\usepackage{centernot}
\usepackage{graphics}
\usepackage[titletoc]{appendix}
\usepackage{fancyhdr}
\usepackage{courier}
\usepackage{ulem}
\usepackage{xcolor}
\usepackage[utf8]{inputenc}
\usepackage{calc}
\usepackage[multiple]{footmisc}
\usepackage[titletoc]{appendix}

\usepackage{tabulary}
\usepackage{booktabs}
\usepackage{float}

\newcommand{\EE}{\mathcal{E}}

\renewcommand{\P}{\mathbb{P}}
\newcommand{\E}{\mathbb{E}}
\newcommand{\R}{\mathbb{R}}
\newcommand{\N}{\mathbb{N}}

\allowdisplaybreaks

\newcommand{\B}{\mathcal{B}}

\newcommand{\RR}{\mathbb{R}}
\newcommand{\PP}{\mathbb{P}}

\newcommand{\cP}{\mathcal{P}}

\newcommand{\MISE}{\mbox{\rm MISE}}

\newcommand{\ME}{ \text{\tiny \it ME}}
\newcommand{\CE}{ \text{\tiny \it CE}}
\newcommand{\vpdp}{v^{\ME}}
\newcommand{\vce}{v^{\CE}}

\def\clap#1{\hbox to 0pt{\hss#1\hss}}

\newtheorem*{proposition*}{Proposition}

\newtheorem*{theorem*}{Theorem}
\newtheorem{theorem}{Theorem}[section]
\newtheorem{corollary}{Corollary}[section]
\newtheorem*{corollary*}{Corollary}
\newtheorem*{definition*}{Definition}
\newtheorem{definition}{Definition}[section]
\newtheorem{lemma}{Lemma}[section]

\newtheorem*{lemma*}{Lemma}
\newtheorem{remark}{Remark}[section]
\newtheorem*{remark*}{Remark}
\newtheorem*{mtheorem*}{Main Theorem}

\usepackage[vlined,linesnumbered,ruled]{algorithm2e}

\SetCommentSty{mycommfont}
\SetKwComment{Comment}{}{}

\numberwithin{equation}{section}

\makeatletter
\def\namedlabel#1#2{\begingroup
    #2%
    \def\@currentlabel{#2}
    \phantomsection\label{#1}\endgroup
}
\newcommand{\specificthanks}[1]{\@fnsymbol{#1}}
\newcommand{\lrarrow}{\mathrel{\mathpalette\lrarrow@\relax}}
\newcommand{\lrarrow@}[2]{%
  \vcenter{\hbox{\ooalign{%
    $\m@th#1\mkern6mu\rightarrow$\cr
    \noalign{\vskip1pt}
    $\m@th#1\leftarrow\mkern6mu$\cr
  }}}%
}
\makeatother
\newcommand*\justify{%
  \fontdimen2\font=0.4em
  \fontdimen3\font=0.2em
  \fontdimen4\font=0.1em
  \fontdimen7\font=0.1em
  \hyphenchar\font=`\-
}

\renewcommand{\texttt}[1]{%
  \begingroup
  \ttfamily
  \begingroup\lccode`~=`/\lowercase{\endgroup\def~}{/\discretionary{}{}{}}%
  \begingroup\lccode`~=`[\lowercase{\endgroup\def~}{[\discretionary{}{}{}}%
  \begingroup\lccode`~=`.\lowercase{\endgroup\def~}{.\discretionary{}{}{}}%
  \catcode`/=\active\catcode`[=\active\catcode`.=\active
  \justify\scantokens{#1\noexpand}%
  \endgroup
}

\title{ Approximation of group explainers with coalition structure using Monte Carlo sampling on the product space of coalitions and features }

\author{Konstandinos Kotsiopoulos \thanks{Emerging Capabilities Research Group, Discover Financial Services Inc., Riverwoods, IL} \textsuperscript{,$\!\!\!$}
\thanks{first author, kostaskotsiopoulos@discover.com, ORCID:0000-0003-2651-0087} \and  Alexey Miroshnikov\textsuperscript{\specificthanks{1},}\thanks{alexeymiroshnikov@discover.com, ORCID:0000-0003-2669-6336} \and
Khashayar Filom\textsuperscript{\specificthanks{1},}\thanks{khashayarfilom@discover.com, ORCID:0000-0002-6881-4460}
\and Arjun Ravi Kannan\textsuperscript{\specificthanks{1},}\thanks{arjunravikannan@discover.com, ORCID:0000-0003-4498-1800}  }

\date{}
    
\begin{document}

\maketitle
\vspace{-0.4 in}
\begin{abstract}
\noindent In recent years, many Machine Learning (ML) explanation techniques have been designed using ideas from cooperative game theory. These game-theoretic explainers suffer from high complexity, hindering their exact computation in practical settings. In our work, we focus on a wide class of linear game values, as well as coalitional values, for the marginal game based on a given ML model and predictor vector. By viewing these explainers as expectations over appropriate sample spaces, we design a novel Monte Carlo sampling algorithm that estimates them at a reduced complexity that depends linearly on the size of the background dataset. We set up a rigorous framework for the statistical analysis and obtain error bounds for our sampling methods. The advantage of this approach is that it is fast, easily implementable, and model-agnostic. Furthermore, it has similar statistical accuracy as other known estimation techniques that are more complex and model-specific. We provide rigorous proofs of statistical convergence, as well as numerical experiments whose results agree with our theoretical findings.
\end{abstract}

\vspace{5pt}

{ \small {\bf keywords}: ML interpretability, cooperative game theory, coalitional values, Monte Carlo sampling}




\section{Introduction}

As the use of Machine Learning (ML) models has become widespread, the need to explain these complex models has also become vital. In the financial industry, predictive models, and strategies based on these models, are subject to federal and state regulations. Regulation B of the Equal Credit Opportunity Act (ECOA, 15 U.S.C. 1691 et seq (1974)) \cite{ECOA} requires a financial institution to notify consumers on the reasons behind certain adverse actions (such as declining a credit application), which in turn requires explaining how attributes in the model contribute to its output.

In recent years, numerous interpretability (or explainability) methods have been developed using concepts from cooperative game theory, motivated by the celebrated work of Shapley \cite{Shapley}. In this setting, given a function $f:\R^n \to \R$, a random vector $X=(X_1,\dots,X_n)$ and an observation $x$, a game $v(\cdot;x,X,f)$ is defined as a set function on subsets of $N=\{1,\dots,n\}$. Here, the predictors $X_1,X_2,\dots,X_n$ are viewed as players participating in the game defined by $v$. Two of the most notable games in the ML literature are the marginal and conditional games and are given respectively by
\[
	\vpdp(S;x,X,f) := \E[ f(x_S,X_{-S}) ], \quad \vce(S;x,X,f) := \E[ f(X) | X_S = x_S ], \ S\subseteq N,
\]
defined in the context of the Shapley value for ML explainability and seen in several works, such as \cite{Strumbelj2011,LundbergLee}.

A game value $h[N,v]=\{h_i[N,v]\}_{i\in N}\in \R^n$ evaluates the contributions of each player to the output of the function $f$. In our work we focus on linear game values $h[N,v]$ that have the form
\[
	h_i[N,v] = \sum_{S\subseteq N\setminus \{i\}} w_i(S,N)(v(S\cup\{i\}) - v(S)), \quad i\in N, \ v\in \{\vpdp,\vce\}.
\]
Explanations generated based on the games in $\{\vpdp,\vce\}$ are called respectively marginal and conditional explanations.

A detailed discussion on how to interpret the game values based on each game has been proposed in \cite{Sundararajan,Janzing,Chen-Lundberg}. Roughly speaking, conditional game values explain predictions $f(X)$ viewed as a random variable, while marginal game values explain the transformations occurring in the model $f(x)$, sometimes called mechanistic explanations \cite{Elton}.

Understanding how a specific model structure and its predictions depend on the underlying predictors is an important part of financial industry regulation, which is why we focus on marginal game values in this paper. The marginal game is straightforward to approximate via the empirical marginal game, that averages the function outputs $f(x_S,x_{-S}^{(k)})$, $x^{(k)}\in \bar{D}_X$, where $\bar{D}_X$ is a background dataset (usually the training set). Hence, the direct estimation of the empirical marginal game value yields a computational complexity of order $O(2^n\cdot |\bar{D}_X|)$, with a statistical accuracy of order $O(|\bar{D}_X|^{-1/2})$. Thus, the biggest challenge in evaluating the empirical marginal game is the exponential number of terms that need to be computed. In addition, when the estimation accuracy is of practical importance, the size of the background dataset also plays a significant role in view of the relatively small convergence rate.

In the literature, there have been several works that propose solutions to the high complexity of estimating a marginal game value in the context of the Shapley value. The Kernel SHAP method approximates marginal Shapley values via a weighted least square problem (the authors work with the conditional game but assume predictor independence, which leads to the marginal game). To make the estimation faster, the method allows discarding of terms from the Shapley formula, thus impacting the estimation accuracy; see \cite{LundbergLee}. Still, the dependence on the number of predictors remains significant, which causes the method to be slow (the authors do not provide the complexity of their method).

TreeSHAP is a model-specific technique, applied to tree ensembles, that attempts to approximate Shapley values. There are two versions of the algorithm, the path-dependent and the interventional TreeSHAP algorithm; see respectively \cite{LundbergTreeSHAP,LundbergNature}. The former version is meant to estimate conditional Shapley values for a given tree ensemble and observation. However, it turns out this approach estimates Shapley values for a tree-based game that differs from both conditional and marginal games. Path-dependent TreeSHAP considers the internal parameters of the model to carry out the estimation, and is known to be implementation non-invariant; see \cite{Filom}. The complexity of the algorithm is $O(T\cdot \mathcal{L}\cdot \mathcal{D}^2)$, where $T$ is the number of trees in the ensemble, $\mathcal{L}$ is the maximum number of leaves, and $\mathcal{D}$ the maximum depth.

Interventional TreeSHAP computes the empirical marginal Shapley value via the use of dynamic programming to obtain polynomial-time performance. The algorithm takes advantage of the tree-based structure and has complexity $O(T\cdot \mathcal{L}\cdot |\bar{D}_X|)$, where $T$,  $\mathcal{L}$ are as above, and $\bar{D}_X$ is the background dataset. One important drawback of the TreeSHAP algorithms is that they are designed for tree-based models and a generalization to a wide class of linear game value is not available.

Another recent work on marginal game values of tree-based models is \cite{Filom}. The paper presents a novel algorithm for computing the marginal Shapley and Banzhaf values of ensembles of oblivious trees. It takes advantage of the internal structure by precomputing probabilities and other parameters associated with the model structure, and uses them to evaluate the empirical marginal Shapley value associated with the training set, via a formula specifically derived for tree-based models consisting of oblivious trees. The complexity of precomputation is $O(T\cdot \mathcal{L}^{\log_2(10)})$, while the complexity of computing the Shapley value for one observation is $O(T\cdot \log_2(\mathcal{L}))$, which is equivalent to the evaluation of a tree ensemble at one instance. A crucial aspect of the method is that, unlike the interventional TreeSHAP algorithm, both precomputations and Shapley value computations are independent of any background dataset.

The work of Castro et al. \cite{Castro2001} suggested a Monte Carlo method for computation of the Shapley value for a generic game $v$. In this method, the Shapley value is viewed as an expected value of an appropriate random variable (associated with the game $v$) defined on the space of coalitions. This allows for an approximation algorithm that samples coalitions, obtains the observation of the random variable and then averages them to obtain the estimate. Since the algorithm is applied to a generic game $v$, given $K$ samples the complexity of the algorithm is $K\cdot O(|v|)$, where the latter is the complexity of evaluating the game $v$, and the statistical accuracy is of order $O(K^{-1/2})$. Directly applying this algorithm to the empirical marginal game with a background dataset of size $K$ yields a high complexity of $O(K^2)$, and the statistical accuracy remains $O(K^{-1/2})$.

Other important work in this direction are the papers of $\rm\check{S}$trumbelj and Kononenko \cite{Strumbelj2011,Strumbelj2014}, which apply the algorithm presented in \cite{Castro2001} and produce a sampling algorithm for the marginal Shapley value (the authors work with the conditional game under assumption of predictor independence). The algorithm jointly samples a coalition and an observation of the predictor vector. This reduces the complexity to $O(K\cdot |f|)$, where $K$ is the number of samples and $|f|$ denotes the complexity of computing the output of $f$ for a given observation. This is a significant improvement over the work of \cite{Castro2001}.

However, the articles \cite{Strumbelj2011,Strumbelj2014} are application-focused and the authors do not provide a rigorous setup for their method. While the authors do provide numerical evidence of convergence on a particular dataset, their work fails to provide proof of convergence and the rate of convergence remains unknown. In recent years, many other game values, and especially coalitional values, have been investigated and applied in the context of ML interpretability, such as the Owen value and two-step Shapley value; see \cite{Lorenzo-Freire,grouppaper}. The works of $\rm\check{S}$trumbelj and Kononenko \cite{Strumbelj2011,Strumbelj2014} focus only on the Shapley value, which limits the use of the approach.

Motivated by the aforementioned works, we design a Monte Carlo-based approach that serves as a generalization to the algorithm of \cite{Castro2001} and \cite{Strumbelj2011,Strumbelj2014}. The main results are as follows:
\begin{itemize}
	\item[($i$)] We design a generalized sampling method to estimate a wide class of marginal linear game values and their quotient game counterparts. In addition, we provide a rigorous setup and statistical analysis of convergence for an appropriate class of models; see Theorems \ref{thm::gamevalue_as_expectation}, \ref{thm::quotientexplainer_as_expectation} and Algorithms \ref{algo_lingamevalue}, \ref{algo_quotgamevalue}.
	
	\item[($ii$)] We extend our method to estimate a wide class of marginal coalitional values $g[N,\vpdp,\cP]$, where $\cP$ is the partition of the predictors. It is common to create the partition $\cP$ based on dependencies, which yields certain advantages; see \cite{grouppaper}. Unlike game values, coalitional values require sampling on the joint space of the triplet consisting of the space of coalitions within a group, the space of group coalitions, and the space of the predictors. Similarly, a rigorous analysis is provided for this sampling method as well; see Theorem \ref{thm::coalvalue_as_expectation} and Algorithms \ref{algo_coalval}, \ref{algo_twostep}.
	
	\item[($iii$)] In certain situations, one may be interested in computing the empirical marginal game value instead of the true value. For this reason we introduce adjusted sampling algorithms that converge to the given empirical marginal game value or coalitional value; see Remarks \ref{rm::emp_marg_lingameval}, \ref{rm::emp_marg_quotgameval}, \ref{rm::emp_marg_coalval}, and \ref{rm::emp_marg_twostep}.
	
	\item[($iv$)] Numerical experiments are conducted that illustrate the convergence rate for various game values and coalitional values on synthetic data examples by estimating the mean integrated squared error (MISE) and relative error. The results show that the numerical evidence agrees with our theoretical results; see Section \ref{subsec::num_experiments}.
\end{itemize}

\noindent {\bf Structure of the paper}. In Section \ref{sec::preliminaries} we provide the notation and necessary background from game theory and probability theory to carry out the analysis. In Section \ref{sec::MC_theory} we introduce a probabilistic framework that allows us to express marginal game values, their quotient game counterparts, and coalitional values, as expected values on appropriate sample spaces. We then present our sampling algorithms together with a rigorous analysis of their statistical error. In Section \ref{sec::numerical} we carry out numerical experiments for various game values and estimate the statistical error for our sampling methods. The results agree with our theoretical findings.

\section{Preliminaries}\label{sec::preliminaries}
\subsection{Notation and hypotheses}

Throughout this article, we consider the joint distribution $(X,Y)$, where $X=(X_1,X_2,\dots,X_n) \in \R^n$ are the predictors and $Y\in \RR$ is a (possibly non-continuous) response variable. Also let $f(x)=\widehat{\E}[Y|X=x]$ be a trained model. We assume that all random variables are defined on the probability space $(\Omega,\mathcal{F},\PP)$, where $\Omega$ is a sample space, $\mathcal{F}$ a $\sigma$-algebra of sets, and $\PP$ a probability measure.

Given a random vector $Z=(Z_1,\dots,Z_k)$ on $(\Omega,\mathcal{F},\P)$, the pushforward probability measure $P_Z$ is a probability measure on $\R^k$  equipped with the $\sigma$-algebra $\B(\R^k)$ of Borel sets of $\R^k$ and satisfying $P_Z(A)=\P(Z \in A)$ for every $A \in \B(\R^k)$. 

Given two probability measures $\mu$ and $\nu$ on $\R^m$ and $\R^k$, respectively, equipped with corresponding Borel $\sigma$-algebras, the product measure $\mu \otimes \nu $ is a probability measure defined on the product $\sigma$-algebra $\B(\R^k)\otimes \B(\R^m)$ that satisfies $\mu \otimes \nu (A \times B)=\mu(A) \cdot \nu (B)$ for each $A \in \B(\R^k)$ and $B \in \B(\R^m)$. The product probability measure is unique but not necessarily complete; the uniqueness of the product measure is a consequence of the measures $\mu$ and $\nu$ being finite (\S 2.5 of \cite{Folland}).

Next, for each $S \subseteq N$ we let $\sigma_S$  denote the $\sigma$-algebra of the product measure $P_{X_S} \otimes P_{X_{-S}}$, where it is assumed that any event evaluated by this measure is first permuted according to the order of indices induced by the pair $(S,-S)$. In particular, for any $f \in L^p(P_{X_S} \otimes P_{X_{-S}},\sigma_S)$ we have
\begin{equation}\label{prod_norm}
\| f \|^p_{L^p(P_{X_S} \otimes P_{X_{-S}},\sigma_S)} := \int |f(x_S,x_{-S})|^p [P_{X_S} \otimes P_{X_{-S}}](dx_{S},dx_{-S})
\end{equation}
where we ignore the variable ordering in $f$ to ease the notation, and we assign $P_{X_{\varnothing}} \otimes P_X = P_X \otimes P_{X_{\varnothing}} = P_X$.

We next define the probability measure 
\begin{equation}\label{marg_game_measure}
\tilde{P}_X:=\frac{1}{2^n} \sum_{S \subseteq N} P_{X_S} \otimes P_{X_{-S}}
\end{equation}
equipped with the $\sigma$-algebra $\tilde{\sigma}=\cap_{S \subseteq N} \sigma_S$. As a consequence, any $\tilde{P}_X$-measurable function is also $P_{X_S} \otimes P_{X_{-S}}$-measurable for each $S \subseteq N$. Finally, the space $L^p(\tilde{P}_X,\tilde{\sigma})$ denotes $\tilde{P}_X$-equivalence classes of functions such that
\begin{equation}\label{marg_game_norm}
\| f \|^p_{L^p(\tilde{P}_X,\tilde{\sigma})} := \frac{1}{2^n} \sum_{S \subseteq N} \int |f(x_S,x_{-S})|^p [P_{X_S} \otimes P_{X_{-S}}](dx_{S},dx_{-S}) < \infty
\end{equation}
where, as before, we ignore the variable ordering in $f$. In what follows, when the context is clear, we will suppress the dependence on the $\sigma$-algebra in the notation of functional spaces.

\subsection{Game-theoretic model explainers}\label{subsec::preliminaries::expl}
Many interpretability techniques have been proposed and utilized over the years, each with their own distinct framework and interpretation. Given their multitude, these techniques can be categorized in various ways. For instance, they can be described as either post-hoc explainers, those that generate feature attributions through the use of the model outputs, or self-interpretable models, which are those whose model structure provides direct information on feature attributions. For more details on these, see \cite{Hall-Gill}.

Some well-known explainability techniques are Partial Dependence Plots (PDP), local interpretable model-agnostic explanations (LIME), explainable Neural Networks (xNN) and Generalized Additive Models plus Interactions (GA$^2$M); see \cite{Friedman}, \cite{Ribeiro et al.}, \cite{Vaughan et al.}, and \cite{Lou2013} respectively. The first two fall under the category of post-hoc explanation techniques, while the latter two fall under the category of self-interpretable models.

Our work will focus on local, post-hoc explanation techniques that have been developed by adapting ideas from game theory. The rest of this section will introduce the relevant concepts and any further material required for the analysis in later sections.

\subsubsection{Games and game values}

In recent years many post-hoc model explainers have been designed by drawing ideas from game theory; see \cite{LundbergLee,Strumbelj2014}. A cooperative game with $n$ players is a set function $v$ that acts on a set of size $n$, say $N \subset \N$, and satisfies $v(\varnothing)=0$. A game value is a map $v \mapsto h[N,v]=(h_1[N,v],\dots,h_n[N,v])\in \R^n$ that determines the worth of each player. In the ML setting,  the features $X \in \R^n$ are viewed as $n$ players with an appropriately designed game $v(S;x,X,f)$ that depends on the given observation $x$ from $P_X$, the random vector $X$ and the trained model $f$. The game value $h[N,v]$, $N=\{1,\dots,n\}$ then assigns the contribution of the respective feature to the total payoff of the game $v(N;x,X,f)$. Non-cooperative games, those that do not satisfy the condition $v(\varnothing)=0$, can also be used in designing model explainers. This requires an  extension of the game value to such games, a topic discussed thoroughly in \cite{grouppaper}.

The game $v(\cdot;x,X,f)$ is a deterministic one, parameterized by the observations $x \in \R^n$. It can be made into a random one by substituting the observation $x$ with the random vector $X$. Such games are out of scope of this paper, and a discussion on them can be found in \cite{grouppaper}.

Two of the most notable (non-cooperative) deterministic games in ML literature are given by
\begin{equation}\label{margcondgames}
    \vce(S;x,X,f)=\E[f(X)|X_S=x_S], \quad \vpdp(S;x,X,f)=\E[f(x_S,X_{-S})],
\end{equation}
with
\begin{equation*}
\vce(\varnothing; x,X, f)=\vpdp(\varnothing; x,X, f)=\E[f(X)],
\end{equation*}
respectively called the conditional and marginal game. These were introduced by \citet{LundbergLee} in the context of the Shapley value \citep{Shapley},
\begin{equation}\label{shapform}
\varphi_i[v] = \sum_{S \subseteq N \backslash\{i\}} \frac{s!(n-s-1)!}{n!} [ v(S \cup \{i\}) - v( S ) ], \quad  s=|S|, \,  n=|N|.
\end{equation}
The Shapley value has garnered much attention from the ML community. From a game-theoretic perspective, \eqref{shapform} represents the payoff allocated to player $i$ from playing the game defined by $v$, while satisfying certain axioms such as symmetry, linearity, efficiency, and the null-player property; see \cite{Shapley} and \cite{grouppaper} for more details. The efficiency property, most appealing to the ML community, allows for a disaggregation of the payoff $v(N)$ into $n$ parts that represent a contribution to the game by each player: $\sum_{i=1}^n \varphi_{i}[v] = v(N)$.

The marginal and conditional games defined in \eqref{margcondgames} are in general not cooperative since they assign $\E[f(X)]$ to $\varnothing$. In such a case, the efficiency property reads $\sum_{i=1}^n \varphi_{i}[v] = v(N)-v(\varnothing)$. See \cite[\S 5.1]{grouppaper} for a careful treatment of game
values for non-cooperative games.

There is a systematic way of extending linear game values so that they can be applied to non-cooperative games as well \cite[\S 5.1]{grouppaper}. In this work, we will mainly consider game values of the form
\begin{equation}\label{lingamevalue}
h_i[N,v]:=\sum_{S\subseteq N\setminus \{i\}}w_i(S,N)\left(v(S\cup\{i\})-v(S)\right), i \in N, \, N \subset \N,
\end{equation}
which are generalizations of the Shapley value, and satisfy the null-player property, that is, they assign zero to players that do not contribute to any coalition; the weighted Shapley value and the Banzhaf value \cite{Banzhaf1965} are other examples of game values of the form \eqref{lingamevalue}. 

A benefit of working with a formula such as \eqref{lingamevalue} is that it remains unchanged after replacing a game 
$v:S\mapsto v(S)$ with the cooperative one  $S\mapsto v(S)-v(\varnothing)$; an extension of this type is called centered. Thus, \eqref{lingamevalue} automatically extends to non-cooperative games such as marginal and conditional ones. Properties such as linearity, symmetry, and null-player property generalize to the case of non-cooperative games in an obvious way except for the efficiency property which should be replaced with $\sum_{i\in N}h_i[N,v]=v(N)-v(\varnothing)$.

We next describe the collection of models $f$ for which the conditional and marginal games, and corresponding game values are well-defined.

\begin{lemma}\label{lmm::game_meas_cond}
Let $f\in [f] \in L^1(P_X)$. Then, for $P_X$-almost sure $x^* \in \R^n$ the conditional game $\vce(\cdot;x^*,X,f)$ is well-defined for any subset of $N$. 
\end{lemma}

\begin{proof}
The proof follows from the definition and properties of the conditional expectation; see \cite[\S 7]{Shiryaev}.
\end{proof}

\begin{lemma}\label{lmm::game_meas_marg}
Let $f\in [f] \in L^p (\tilde{P}_X)$. Then:

\begin{itemize}
    \item [$(i)$] There exists $P_X$-measurable set $U^{(f)}_p \in \B(\R^n)$ such that $P_X(U^{(f)}_p)=1$ and for each $x^{*} \in U_p^{(f)}$ and each $S \subsetneq N$ the map $x_{-S} \mapsto f(x_S^*,x_{-S})$ is $P_{X_{-S}}$-measurable and satisfies
\[
\int |f(x^*_S,x_{-S})|^p P_{X_{-S}}(dx_{-S}) < \infty. 
\]
\item [$(ii)$] For $P_X$-almost sure $x^* \in \R^n$ the marginal game $\vpdp(\cdot;x^*,X,f)$ is well-defined for any subset of $N$. 
\end{itemize}
\end{lemma}
\begin{proof} The proof follows from the Fubini theorem \cite[\S 2.5]{Folland} and the fact that $\tilde{\sigma} \subseteq \sigma_S \subseteq \B(\R^n)$.
\end{proof}

\begin{corollary}\label{corr::game_val_exist}
Let $h[N,\cdot]$ be any linear game value as in \eqref{lingamevalue} on $N$.
\begin{itemize}
    \item[(i)] Let $f\in [f] \in L^1(P_X)$. Then, $h[N,\vce(\cdot;x^*,X,f)]$ is well-defined for $P_X$-almost sure $x^*$.
    \item[(ii)] Let $f\in [f] \in L^1(\tilde{P}_X)$. Then, $h[N,\vpdp(\cdot;x^*,X,f)]$ is well-defined for $P_X$-almost sure $x^*$.
\end{itemize}
\end{corollary}

As discussed in \cite{Chen-Lundberg,Janzing,Sundararajan}, as well as in our previous work \cite{grouppaper}, the interpretation of game values is based on the corresponding game used to construct them. Game values for the conditional game, or conditional game values, explain the output of the random variable $f(X)-\E[f(X)]$, meaning that they take into account the joint distribution $P_X$. On the other hand, marginal game values explain the output of the function $f(x)-\E[f(X)]$, meaning they take into account how the model structure utilized the predictors.

In practice, given a background dataset $\bar{D}_X=\{x^{(1)},\dots,x^{(K)}\}$, $\vpdp(S;x,X,f)$ can be approximated by the empirical marginal game given by
\begin{equation}\label{empiricalmarggame}
    \hat{v}^{\ME}(S;x,\bar{D}_X,f):=\frac{1}{|\bar{D}_X|}\sum_{x'\in \bar{D}_X} f\left(x_S,x'_{-S}\right).
\end{equation}
We next show that for $P_X$-almost sure observations $x^* \in \R^n$ the empirical marginal game value is the unbiased point estimator of the marginal game value with the mean squared error bounded by $O(|\bar{D}_X|^{-1})$  as shown in the following lemma.

\begin{lemma}\label{lmm::error_emp_marg_estimator}
Let $f\in [f] \in L^2(\tilde{P}_X)$ and $h[N,\cdot]$ as in \ref{lingamevalue}. Let $\bar{\bf D}_X=\{X^{(1)},X^{(2)},\dots,X^{(K)}\}$ be a background dataset containing independent random samples from $P_X$. Then for $P_X$-almost sure $x^* \in \R^n$
\[
\E \big[|h[N,\vpdp(\cdot;x^*,X,f)]-h[N,\hat{v}^{\ME}(\cdot;x^*,\bar{\bf D}_X,f)]|^2 \big] \leq \frac{2}{K} \Big( \sum_{S \subseteq N \setminus \{i\}} w^2_i(S,n) \Big) \cdot  \Big( \sum_{S \subseteq N}  Var(f(x^*_S,X_{-S})) \Big)<\infty.
\]
As a consequence, the Mean Integrated Squared Error (MISE) satisfies
\[
\E_{x^* \sim P_X} \Big[ \E \big[|h[N,\vpdp(\cdot;x^*,X,f)]-h[N,\hat{v}^{\ME}(\cdot;x^*,\bar{\bf D}_X,f)]|^2 \big] \Big]
\leq \frac{2}{K} \Big( \sum_{S \subseteq N \setminus \{i\}} w^2_i(S,n) \Big) \cdot  \|f\|^2_{L^2(\tilde{P}_X)}.
\]
\end{lemma}

\begin{proof}
The proof follows from Lemma \eqref{lmm::game_meas_marg}, equations \eqref{margcondgames}, \eqref{lingamevalue}, and \eqref{empiricalmarggame}, and the use of Cauchy-Schwartz.
\end{proof}

Note that evaluating \eqref{lingamevalue} for a given game $v$ can be computationally intensive. The complexity is proportional to $2^n$ times the complexity of computing the game $v$. Thus, when utilizing the empirical marginal game, the complexity becomes $O(2^n \cdot |\bar{D}_X|)$. One can therefore reduce the complexity by taking a small dataset. However, there will be a trade-off with the estimation accuracy as discussed in Lemma \ref{lmm::error_emp_marg_estimator}.

Several works have proposed approximation techniques with the aim of reducing that complexity, such as Kernel SHAP in \cite{LundbergLee} and TreeSHAP in \cite{LundbergTreeSHAP}. We will discuss these in further detail below.

\subsubsection{Grouping and group explainers}

Our work in \cite{grouppaper} presents another avenue for reducing the high complexity of \eqref{lingamevalue} and generating contributions that are stable. Instead of attempting to explain the contribution of each individual predictor, one can first define a partition $\cP = \{S_1, S_2,\dots,S_m\}$ of predictors based on dependencies and then explain the contribution of each group to the model output. Such explainers are called group explainers. To form the partition $\cP$ in practice, \cite{grouppaper} presents a mutual information-based hierarchical clustering algorithm. The dissimilarity measure in this algorithm is designed using the Maximal Information Coefficient (MIC); for more details on this measure of dependence, see \cite{Reshef11,Reshef16b,Reshef16}.

There are three types of group explainers presented in \cite{grouppaper}: trivial group explainers, quotient game explainers, and explainers based on games with coalition structure. In this work we will focus on the latter two since trivial explainers are based on single feature explanations that do not incorporate the group structure in the calculation.

\begin{definition}\label{def::quotientgame}
    Let $X\in \RR^n$ be a vector of predictors, $\cP = \{S_1,\dots,S_m\}$ a partition of $N$ and $v$ a game. The quotient game $(M,v^{\cP})$ is defined as
    \[
        v^{\cP}(A):=v\Big(\bigcup_{j\in A}S_j\Big), \quad A\subseteq M=\{1,2,\dots,m\}.
    \]
\end{definition}

The game $v^{\cP}$ suggests that the players playing the game are the predictor groups, thus the payoff is allocated to groups rather than single predictors in the context of \eqref{shapform}. Game values based on $v^{\cP}$ are called quotient game explainers.

\begin{definition}\label{def::quotientexpl}
    Given a partition $\cP$ as in Definition \ref{def::quotientgame}, a trained model $f$, and a game value $h$ based on a game $v \in \{\vce,\vpdp\}$, the  marginal and conditional quotient game explainers $h^{\cP}$ at the observation $x^* \in \R^n$ based on $(h,\cP)$ are given by
    \[
        h_{S_j}^{\cP}(x^*;X,f,v):=h_j[M,v^{\cP}(\cdot;x^*,X,f)], \quad S_j\in \cP, \, v \in \{\vce,\vpdp\}.
    \]
    The coefficients of $h_{S_j}^{\cP}$ are denoted by $w_j(A,M)$, $A\subseteq M\setminus \{j\}$.
\end{definition}

\begin{remark}\rm
It follows from Lemma \ref{lmm::game_meas_cond} that $h_{S_j}^{\cP}(x^*;X,f,\vce)$ is well-defined for $P_X$-almost sure $x^*$ if $f \in [f] \in L^1(P_X)$. Similarly, from Lemma \ref{lmm::game_meas_marg} it follows that $h_{S_j}^{\cP}(x^*;X,f,\vpdp)$ is well-defined for $P_X$-almost sure $x^*$ by requiring $f \in [f] \in L^1(\tilde{P}_X)$; this requirement, in fact, can be relaxed by assuming that $f \in [f] \in L^1(\tilde{P}_{X,\cP})$ where \[
\tilde{P}_{X,\cP} := \frac{1}{2^m} \sum_{A \subseteq M, Q_A=\cup_{\alpha \in A}S_{\alpha}} P_{X_{Q_A}} \otimes P_{X_{-Q_A}}.
\] 
We note that if the group predictors $(X_{S_1},\dots,X_{S_m})$ are independent then $\tilde{P}_{X,\cP}=P_X$ and $v^{\ME,\cP}=v^{\CE,\cP}$.
\end{remark}

Two examples of such an explainer are the conditional and marginal quotient game Shapley values $\varphi[v^{\CE,\cP}]$ and $\varphi[v^{\ME,\cP}]$ respectively, which utilize the games in \eqref{margcondgames} and a given partition $\cP$.

\begin{equation}\label{quotientShapley}
    \varphi_j[v^{\cP}]=\sum_{A \subseteq M \setminus \{j\}} \frac{|A|!(m-|A|-1)!}{m!} \left[ v^{\cP}(A \cup \{j\}) - v^{\cP}( A ) \right], \quad S_j\in \cP, \quad v\in \{\vce,\vpdp\}.
\end{equation}

Explainers based on games with coalition structure, or coalitional values, utilize the group partition $\cP$ in order to output contributions for individual predictors. Games with coalitions were introduced in \cite{Aumann1974}.

\begin{definition}\label{def::coalvalue}
    Given $N$, $\cP$, $v$ as in Definition \ref{def::quotientgame}, a coalitional value $g$ is a map of the form $g[N,v,\cP]=\{g_i[N,v,\cP]\}\in \RR^n$, where $g_i$ represents the contribution of the player $i \in S_j$, $S_j \in \cP$ to the total payoff $v(N)$.
\end{definition}

\begin{remark}\rm
Similar to game values, it follows from Lemma \ref{lmm::game_meas_cond} that $g[N,\vce(\cdot;x^*,X,f),\cP]$ is well-defined for $P_X$-almost sure $x^*$ if $f \in [f] \in L^1(P_X)$. From Lemma \ref{lmm::game_meas_marg} it follows that $g[N,\vpdp(\cdot;x^*,X,f),\cP]$ is well-defined for $P_X$-almost sure $x^*$ by requiring $f \in [f] \in L^1(\tilde{P}_X)$.
\end{remark}

Two notable coalitional values are the Owen $Ow[N,v,\cP]$ and the Banzhaf-Owen value $BzOw[N,v,\cP]$; see \cite{Owen,Owen1982,Lorenzo-Freire}. Both are rather complex quantities, as seen by their respective elements below. Suppose we are interested in the contribution of $X_i$ to the model output such that $X_i$ forms a union $S_j$ with other predictors. Its Owen and Banzhaf-Owen values will be given by

\begin{equation}\label{OwenandBzOw}
    \begin{aligned}
        Ow_i[N,v,\cP] &= \sum_{A\subseteq M\setminus \{j\}}\sum_{T\subseteq S_j\setminus \{i\}} \tfrac{|A|!(m-|A|-1)!}{m!}\tfrac{|T|!(|S_j|-|T|-1)!}{|S_j|!} \left[ v(Q_A \cup T \cup \{i\}) - v(Q_A \cup T) \right]\\
        BzOw_i[N,v,\cP] &= \sum_{A\subseteq M\setminus \{j\}}\sum_{T\subseteq S_j\setminus \{i\}} \tfrac{1}{2^{m-1}}\tfrac{1}{2^{|S_j|-1}} \left[ v(Q_A \cup T \cup \{i\}) - v(Q_A \cup T) \right]
    \end{aligned}
\end{equation}
where $Q_A=\cup_{\alpha \in A}S_{\alpha}$, $i\in S_j$. As with the Shapley value, if $v=\vce$ then we call the quantities in \eqref{OwenandBzOw} conditional Owen values and conditional Banzhaf-Owen values. Respectively, for $v=\vpdp$ we say marginal Owen and marginal Banzhaf-Owen values.

Compare the complexity $O(2^n)$ of the Shapley value with those of \eqref{quotientShapley} and \eqref{OwenandBzOw}, which are respectively $O(2^m)$ and $O(2^{|S_j|+m})$. Although lower, in practice these complexities may still not suffice for an exact computation to be carried out. For example, it may be that there are $n=200$ predictors and $m=30$ groups are formed. The complexity $O(2^{30})$ is still tremendously large. This motivates us to design approximation techniques for group explainers that are relatively fast and accurate. In Section \ref{sec::MC_theory}, we design a Monte Carlo (MC) sampling algorithm for group explainers and provide the error analysis, showing also that the resulting estimator is consistent and unbiased.

In addition to reducing the computational complexity, there are several benefits in using group explainers for ML explainability over standard game values, as discussed in detail in \cite{grouppaper}. By forming independent groups of predictors, the marginal and conditional games coincide, allowing to unify their interpretations. Thus, one can consider both the model structure and the joint distribution of $X$ when generating explanations. This leads to more stable contribution values across different models trained on the same dataset.

\subsection{Relevant works}\label{subsec::preliminaries::relwork}
Many approximation techniques for \eqref{shapform} have been developed over the years. Lundberg and Lee \cite{LundbergLee} introduced Kernel SHAP, a method that approximates marginal Shapley values by solving a weighted least square minimization problem. The paper however does not provide any asymptotics. Lundberg et al. \cite{LundbergTreeSHAP} designed an algorithm called TreeSHAP that is specific to tree-based models. One of the parameters of the algorithm allow the user to switch between two games, the marginal (referred to as interventional) and a tree-based game that is defined by the tree structure of the model. The latter game allows for a fast computation of the corresponding Shapley values. However, the tree-based game is not implementation invariant, and as a consequence TreeSHAP values for this game do not approximate the conditional Shapley values; see \cite{Filom}. An error analysis has not been carried out for this method either.

Castro et al. \cite{Castro2001} suggested a Monte Carlo method for computing the Shapley value for a generic game $v$. The idea is to view the Shapley value as an expected value of an appropriate random variable (associated with the game $v$) defined on the space of coalitions. Since the algorithm is applied to a generic game $v$, the complexity of the algorithm is $K\cdot O(|v|)$, where $K$ are the number of samples and $|v|$ is the complexity of evaluating the game $v$, with the statistical accuracy being of order $O(K^{-1/2})$. Directly applying this algorithm to the empirical marginal game with a background dataset of size $K$ yields a high complexity of $O(K^2)$, and the statistical accuracy remains $O(K^{-1/2})$.

Utilizing the algorithm from \cite{Castro2001}, $\rm\check{S}$trumbelj and Kononenko \cite{Strumbelj2011,Strumbelj2014} presented a sampling algorithm to approximate the marginal Shapley value with a reduced complexity of $O(K\cdot |f|)$, with $|f|$ the computational complexity of evaluating the model $f$ at a given observation, and provide numerical evidence showing convergence for their algorithm. The authors, however, do not provide a rigorous setup for their method, nor do they explicitly define the marginal game. They define the conditional game and subsequently assume predictor independence, leading to the marginal game. For more techniques that allow for approximating the Shapley value, please see \cite{Chen2022}.

Our analysis is inspired by the works of \cite{Castro2001,Strumbelj2011,Strumbelj2014} and seeks to rigorously generalize the sampling approach for a wider class of linear game values, as well as coalitional values. To our knowledge, there has not been any work on designing approximation techniques for quotient game and coalitional explainers, and we believe this work constitutes the first endeavor in this direction.

\section{Monte Carlo sampling for game values}\label{sec::MC_theory}

The two games defined in \eqref{margcondgames} are both noteworthy and have seen extended use in the ML explainability literature when utilizing the corresponding game values. In this section we design Monte Carlo approximation techniques for marginal game values and provide the necessary theoretical backing for these approximations. The marginal game is a natural choice for our approach, which approximates marginal game values in a fast, accurate and model-agnostic manner. Approximation techniques have been designed for conditional game values; see \cite{aas2021}. We start with linear game values of the form \eqref{lingamevalue} and then generalize our approach to quotient game and coalitional game values.

\subsection{Sampling for marginal linear game values}

When utilizing the marginal game, observe that the linear game value formula \eqref{lingamevalue} consists of $2^n$ terms and each term is an expectation. Naively replacing these expectations with an empirical average over a background dataset $\bar{D}_X$ would lead to the  computation of the empirical marginal game value that has  computational complexity of $O(2^n \cdot |\bar{D}_X|)$ and approximates the marginal one with the statistical estimation accuracy of $O(|\bar{D}_X|^{-1/2})$ as indicated in Lemma \ref{lmm::error_emp_marg_estimator}. This would be tremendously high for practical purposes.

The key in our analysis is to observe that the game value element $h_i[N,\vpdp]$ can be expressed as an expected value of an appropriate function with respect to a product measure which is defined on the joint space of predictors and coalitions.

For our analysis we make the following assumptions for the coefficients $w_i(S,N)$ of the game value:
\begin{enumerate}
    \renewcommand{\labelenumi}{\textbf{(\theenumi)}}
    \renewcommand{\theenumi}{A\arabic{enumi}}

    \item\label{hyp::coeff_nonneg} $w_i(S,N)\ge 0$, for any $S\subseteq N\setminus \{i\}$.
    \item\label{hyp::coeff_sumtoone} $\displaystyle \sum_{S\subseteq N\setminus \{i\}}w_i(S,N)=1$.
\end{enumerate}
In principle, the first assumption would suffice since one can normalize the coefficients and scale $h_i[N,\vpdp]$ by $\sum_{S\subseteq N\setminus \{i\}}w_i(S,N)$. By assuming both \eqref{hyp::coeff_nonneg} and \eqref{hyp::coeff_sumtoone} the coefficients become probabilities, in turn defining the following collection of probability measures which simplify our analysis.

\begin{definition}\label{def::P_h}
    Suppose \eqref{hyp::coeff_nonneg}-\eqref{hyp::coeff_sumtoone} hold for the coefficients of a game value $h$ whose elements are as in \eqref{lingamevalue}. For each $i \in N$ we define the discrete probability measure $P_i^{(h)}$ defined on subsets of $N \setminus \{i\}$ satisfying $P_i^{(h)}(S) = w_i(S,N)$ for each $S \subseteq N\setminus \{i\}$.
\end{definition}

\begin{remark}\rm\label{rem::shapley_probs}
    The probability $P_i^{(h)}(S)$ can be viewed as the probability of selecting the coalition $S$ under certain conditions. For instance, if one assumes that distinct sizes of $ S\subseteq N \setminus \{i\}$ are equally likely to occur and that given a specific size all distinct sets $S$ are equally likely, we have $P_i^{(h)}(S)=\frac{|S|!(n-|S|-1)!}{n!}$ as in \eqref{shapform}.
\end{remark}

First, we state an auxiliary lemma that represents a game value as an expectation of the random variable defined on the set of coalitions.
\begin{lemma}\label{game_val_expect_prelim}
Let $N,v,h[N,v]$ be as in \eqref{lingamevalue} and $\{P_i^{(h)}\}_{i \in N}$ as in Definition \ref{def::P_h}. Then
\begin{equation}\label{exp_repr_game_prelim}
h_i[N,v] = \int \big(v( S\cup \{i\} )-v( S )\big) P_i^{(h)}(dS), \quad i \in N.
\end{equation}
\end{lemma}

The representation \eqref{exp_repr_game_prelim} allows one to approximate $h_i[N,v]$ using an MC method that involves sampling from  $P^{(h)}_i$ a sequence of coalitions $S^{(k)} \subseteq N \setminus \{i\}$, $k \in \{1,\dots,K\}$, evaluating the difference $v( S^{(k)} \cup \{i\} )-v(S^{(k)})$ and then averaging the result. This approach, however, is not optimal when one deals with  the marginal game $\vpdp$ because $\vpdp(S^{(k)};x^*,X,f)$ itself is an expectation of $f$ with respect to $P_{X_{-S^{(k)}}}$, which in practice is replaced with the appropriate average over a background dataset $\bar{D}_X$. In this case, such an MC procedure would yield an MC estimator of the empirical marginal coalitional value $h_i[N,\hat{v}^{\ME}(\cdot;x^*,\bar{D}_X,f)]$ with computational complexity  $O(K \cdot |\bar{D}_X|)$. When one is interested in the approximation of the true marginal value, it is optimal (from the perspective of the error-complexity trade-off) to choose $K=|\bar{D}_X|$, in which case the complexity becomes $O(K^2)$ and the obtained value estimates the true marginal game value with the mean squared error $O(K^{-1})$ according to Lemma \ref{lmm::error_emp_marg_estimator}.

This discussion motivates us to adjust the above result with the goal of reducing the complexity from $O(K^2)$ to $O(K)$ with $K=|\bar{D}_X|$. To this end, for each $i\in N$ we view coalitions $S\subseteq N\setminus \{i\}$ and observations $x$ from $P_X$ as pairs coming from the product space $P^{(h)}_i \otimes P_X$. With this setup one can write the marginal game value $h_i[N,\vpdp(\cdot;x^*,X,f)]$ as an expected value with respect to a product of two probability measures as shown below.

\begin{definition}
Let $f\in [f] \in L^p(\tilde{P}_X)$ and $U^{(f)}_p \in \B(\R^n)$ be a set of $P_X$-probability $1$ as in Lemma \ref{lmm::game_meas_marg}. Set
\begin{equation}\label{param_marg_norm}
\nu_{x^*}^{(p)}(f) := \sum_{S \subseteq N} \int |f(x^*_S,x_{-S})|^p P_{X_{-S}}(dx_{-S}) < \infty, \quad  x^* \in U_p^{(f)}.
\end{equation}
\end{definition}

\begin{theorem}\label{thm::gamevalue_as_expectation} Let $f\in [f] \in L^p(\tilde{P}_X)$, $p \geq 1$. Let $h[\cdot,N]$ be as in \eqref{lingamevalue}, and suppose that the coefficients $w_i(S,N)$ of $h_i$, $i\in N$, satisfy \eqref{hyp::coeff_nonneg}-\eqref{hyp::coeff_sumtoone}. 

\begin{itemize}
    \item [(i)] For $P_X$-almost sure $x^* \in \R^n$, the map
\begin{equation}\label{Delta_def}
(S,x) \mapsto \Delta_i(S,x; x^*,f):=f(x^*_{S\cup \{i\}},x_{-S\cup \{i\}})-f(x^*_S,x_{-S})
\end{equation}
is $ P_i^{(h)} \otimes P_X $-measurable and $p$-power integrable satisfying the bound
\begin{equation}\label{Delta_bound}
\int |\Delta_i(S,x;x^*,f)|^p  \, [P_i^{(h)} \otimes P_X](dS,dx) \leq  2^p \cdot \bar{w}_{i,N} \cdot \nu_{x^*}^{(p)}(f), \quad \bar{w}_{i,N}:=\max_{S \subseteq N \setminus \{i\}} w_i(S,N).
\end{equation}

\item [(ii)] For $P_X$-almost sure $x^* \in \R^n$ the game value $h_i[N,\vpdp(\cdot ; x^*, X, f)]$ is well-defined and can be written as an expectation with respect to  the product probability measure $P_i^{(h)}\otimes P_X$. Specifically,
\begin{equation}\label{eq::h_as_expectation}
    h_i[N,\vpdp(\cdot; x^*, X,f)]=\int \Delta_i(S,x;x^*,f)\,[P_i^{(h)} \otimes P_X](dS,dx).
\end{equation}

Consequently, taking random samples $(\mathcal{S}^{(k)},X^{(k)})\sim P_i^{(h)}\otimes P_X$, $k \in \{1,2,\dots,K\}$, we have 
    \begin{equation}\label{eq::h_approx}
        \frac{1}{K}\sum_{k=1}^K \Delta_i(\mathcal{S}^{(k)}, X^{(k)};x^*,f) = h_i[N,\vpdp(\cdot ; x^*,X,f)] + \mathcal{E}_K,
    \end{equation}
    where $\mathcal{E}_K\to 0$ in probability as $K\to \infty$. In addition, if $f \in [f] \in L^2(\tilde{P}_X)$, then $\mathcal{E}_K \to 0$ in $L^2(\P)$ with rate $O(K^{-1/2})$.
\end{itemize}
\end{theorem}

\begin{proof}
Since $f \in [f] \in L^p(\tilde{P}_X)$ it follows from Lemma \ref{lmm::game_meas_marg} that for $P_X$-almost sure $x^*$ and each $S^* \subset N$ the map $x \mapsto \Delta_i(S^*,x; x^*,f)$ is $P_X$-measurable. Since $P_i^{(h)}$ is discrete, we conclude that the map $(S,x) \mapsto \Delta_i(S,x; x^*,f)$ is measurable with respect to $P_i^{(h)} \otimes P_X$. Furthermore, we have
\begin{equation}
\begin{aligned}
\int |\Delta_i(S,x;x^*,f)|^p  \, [P_i^{(h)} \otimes P_X](dS,dx) &= \sum_{S\subseteq N} w_i(S,N) \int |f(x^*_{S\cup \{i\}},x_{-S\cup \{i\}})-f(x^*_S,x_{-S})|^p P_X(dx) \\
\leq 2^p &\cdot \bar{w}_{i,N}  \cdot \sum_{S\subseteq N} \int |f(x^*_S,x_{-S})|^p P_X(dx) = 2^p \cdot \bar{w}_{i,N}  \cdot \nu_{x_*}^{(p)}(f),
\end{aligned}
\end{equation}
which proves \eqref{Delta_bound}.

Since $f \in [f] \in L^1(\tilde{P}_X)$, by Lemma \ref{lmm::game_meas_marg} there exists a Borel set $U^{(f)}_1$ of $P_X$-probability $1$ such that the marginal game $\vpdp(\cdot;x^*,X,f)$ and the game value $h[N,\vpdp(\cdot;x^*,X,f)]$ are well-defined for all $x^* \in U_1^{(f)}$.

Let $x^* \in U^{(f)}_1$. Recalling that 
\[
    \vpdp( S; x^*,X,f ) = \E\left[ f(x^*_S,X_{-S}) \right]
    =\int f(x^*_S,X_{-S}(\omega))\,\P(d\omega) =\int f(x^*_S,x_{-S})\,P_{X_S}(dx_{S}) ,
\]
we obtain
    \begin{align*}
        &h_i[N,\vpdp(\cdot;x^*,X,f)] = \sum_{S \subseteq N \setminus \{i\}} w_i(S,N) \left[ \vpdp(S \cup \{i\};x^*,X,f) - \vpdp( S;x^*,X,f ) \right]\\
        &=\int \left[ \vpdp(S \cup \{i\}; x^*,X,f) - \vpdp( S;x^*,X,f ) \right]P_i^{(h)}(dS)\\
        &=\int\int f(x^*_{S\cup\{i\}},x_{-(S\cup \{i\})}) P_{X_{-(S \cup \{i\})}}(dx_{-(S \cup \{i\})}) P_i^{(h)}(dS)  - \int\int f(x^*_S,x_{-S}) P_{X_S}(dx_S) P_i^{(h)}(dS) .
    \end{align*}
Since the integrands in the first and second terms of the last expression do not explicitly depend on $X_{S\cup \{i\}}$ and $X_S$, respectively, we can write
    \begin{align*}
        h_i[N,\vpdp]
        &=\int\int f(x^*_{S\cup\{i\}},x_{-(S\cup \{i\})})P_i^{(h)}(dS)P_X(dx) - \int\int f(x^*_S,x_{-S}) P_i^{(h)}(dS)P_X(dx).
    \end{align*}
Putting the two integrals together we obtain \eqref{eq::h_as_expectation}.

For \eqref{eq::h_approx}, we set 
\[
\displaystyle \mathcal{E}_K := \frac{1}{K}\sum_{k=1}^K \Delta_i(\mathcal{S}^{(k)},X^{(k)};x^*,f) - h_i[N,\vpdp(\cdot;x^*,X,f)]
\]
and observe that $\E[\mathcal{E}_K]=0$. Since the samples $(\mathcal{S}^{(k)},X^{(k)})$ are independent and identically distributed, the average converges in probability to $h_i[N,\vpdp(\cdot;x^*,X,f)]$ by the weak law of large numbers. 

Suppose now that $f \in [f] \in L^2(\tilde{P}_X)$. Then for the error rate, we have
\[
    \|\mathcal{E}_K\|_{L^2(\P)}^2=Var\left[\frac{1}{K}\sum_{k=1}^K \Delta_i(\mathcal{S}^{(k)},X^{(k)};x^*,f)\right]=\frac{1}{K}Var[\Delta_i(\mathcal{S}^{(1)},X^{(1)};x,f)]\le \frac{4}{K} \cdot \bar{w}_{i,N} \cdot \nu^{(2)}_{x_*}(f),
\]
which implies that $\mathcal{E}_K$ converges to zero in $L^2(\P)$ at a rate of $1/\sqrt{K}$.
\end{proof}

Recall that $\Delta_i(\mathcal{S},X;x,f)$ having finite variance is not a necessary condition for the weak law of large numbers. Convergence still occurs but at a slower rate. In either case, equation \eqref{eq::h_approx} informs us on how to approximate the game value $h[N,\vpdp]$ in practice. This leads us to design the following sampling algorithm:

\begin{algorithm}
    \SetAlgoLined 
    \KwIn{Observation $x^*\in \R^n$, model $f$, dataset $\bar{D}_X=\{x^{(k)}\}_{k=1}^K$.
    }
    \KwOut{MC estimate $\hat{h}[N,\vpdp]$ of $h[N,\vpdp]$.
    }
    \BlankLine
    \For{$i$ in $\{1,\dots,n\}$}
    {
        $\hat{h}_i := 0$;\\
        \For{$k$ in $\{1,\dots,K=|\bar{D}_X|\}$}
        {
            Select the observation $x^{(k)}$ from $\bar{D}_X$;\\
            Draw a random coalition $S^{(k)}$ from the distribution $P_i^{(h)}$;\\
            $\hat{h}_i :=\hat{h}_i + \Delta_i(x^{(k)},S^{(k)};x^*,f)$;
        }
        $\hat{h}_i := \hat{h}_i/K$;
    }
    \Return{$\hat{h}[N,\vpdp] := \big(\hat{h}_1,\dots,\hat{h}_n\big)$;}
\caption{Monte Carlo sampling for linear game values}\label{algo_lingamevalue}\label{algo_game_marg}
\end{algorithm}

\begin{remark}\rm\label{rm::emp_marg_lingameval}
In some applications one may be interested in computing the empirical marginal game value $h[N,\hat{v}^{\ME}(\cdot;x^*,\bar{D}_X,f)]$ whose computational complexity is $O(2^n \cdot |\bar{D}_X|)$. In this case, even if the background dataset is small, computing the empirical game value directly becomes infeasible for large $n$. One can then estimate the empirical marginal game value using an adjusted version of Algorithm \ref{algo_game_marg}. The adjustments are as follows. The input to the algorithm must contain the number $\tilde{K}$ of MC iterations, which is now independent of the number of samples in the background dataset $\bar{D}_X$. Line 3 must contain the for-loop over $k \in \{1,\dots,\tilde{K}\}$ and line 4 must be replaced with ``Draw a random observation $x^{(k)}$ from the empirical distribution $P_{\bar{D}_X}$''.

The value produced by the adjusted Algorithm \ref{algo_game_marg} is the estimation of $h[N,\hat{v}^{\ME}(\cdot;x^*,\bar{D}_X,f)]$ with a mean squared error of estimation proportional to $O(\tilde{K}^{-1})$, which is independent of  $|\bar{D}_X|$. Since the number $\tilde{K}$ of MC samples is not limited by $|\bar{D}_X|$, one can achieve arbitrary estimation accuracy by indefinitely sampling from $P_i^h \otimes P_{\bar{D}_X}$. Finally, we note that the difference between the empirical marginal game value and the true marginal game value in $L^2(\PP)$ is bounded by $O(|\bar{D}_X|^{-1/2})$ as discussed in Lemma \ref{lmm::error_emp_marg_estimator}.
\end{remark}

\subsection{Sampling for marginal quotient game values}

As mentioned in Section \ref{sec::preliminaries}, group explainers can have reduced complexity compared to that of linear game values, but it would still be considerably high for any practical implementation. Following the analysis presented above for linear game values, we propose a similar theorem and algorithm for quotient game values.

We make the same assumptions for the coefficients $w_j(A,M)$ of $h_{S_j}^{\cP}$, $j\in M$, as before. Thus, stating that \eqref{hyp::coeff_nonneg}-\eqref{hyp::coeff_sumtoone} hold for the coefficients of $h_{S_j}^{\cP}$ we mean that $w_j(A,M)\ge 0$ for any $A\subseteq M\setminus \{j\}$, and $\sum_{A\subseteq M\setminus \{j\}}w_j(A,M)=1$. This leads to the following definition.

\begin{definition}\label{def::P_hP}
    Let $\cP=\{S_1,\dots,S_m\}$. Suppose \eqref{hyp::coeff_nonneg}-\eqref{hyp::coeff_sumtoone} hold for the coefficients of a quotient game value element $h_{S_j}^{\cP}$, as given in Definition \ref{def::quotientexpl}. For each $j \in M$ we define the discrete probability measure $P_j^{(h,\cP)}$ that satisfies $P_j^{(h,\cP)}(A) = w_j(A,M), A \subseteq M\setminus \{j\}$.
\end{definition}

\begin{definition}
Let $f\in [f] \in L^p(\tilde{P}_X)$, $U^{(f)}_p \in \B(\R^n)$ be a set of $P_X$-probability $1$ as in Lemma \ref{lmm::game_meas_marg}, and $\cP=\{S_1,\dots,S_m\}$ a partition of $N$. Set
\begin{equation}\label{param_marg_norm_quot}
\nu_{x^*}^{(p)}(f,\cP) := \sum_{A \subseteq M, Q_A=\cup_{\alpha \in A} S_{\alpha}} \int |f(x^*_{Q_A},x_{-Q_A})|^p P_{X_{-Q_A}}(dx_{-Q_A}) < \infty, \quad  x^* \in U_p^{(f)}.
\end{equation}
\end{definition}

We now state the equivalent of Theorem \ref{thm::gamevalue_as_expectation} for quotient game values:

\begin{theorem}\label{thm::quotientexplainer_as_expectation}
    Let $f \in [f] \in L^p(\tilde{P}_X)$, $p \geq 1$. Let $\cP$ and $h^{\cP}$ be as in Definition \ref{def::quotientexpl}, and suppose that the coefficients $w_j(A,M)$ of $h_{S_j}^{\cP}$, $j\in M$, satisfy \eqref{hyp::coeff_nonneg}-\eqref{hyp::coeff_sumtoone}.

\begin{itemize}
    \item [$(i)$] For $P_X$-almost sure $x^* \in \R^n$, the map
\begin{equation}\label{Delta_quot_def}
(A,x) \mapsto \Delta_j(A,x;x^*,f):=f(x^*_{\cup_{\alpha \in A\cup \{j\}}S_{\alpha}},x_{-\cup_{\alpha \in A\cup \{j\}}S_{\alpha}})-f(x^*_{\cup_{\alpha \in A}S_{\alpha}},x_{-\cup_{\alpha \in A}S_{\alpha}}),
\end{equation}
where $A \subseteq M \setminus \{i\}$, is $ P_j^{(h,\cP)} \otimes P_X $-measurable and $p$-power integrable satisfying the bound
\begin{equation}\label{Delta_bound_quot}
\int |\Delta_j(A,x;x^*,f)|^p  \, [P_j^{(h,\cP)} \otimes P_X](dA,dx) \leq  2^p \cdot \bar{w}_{j,M} \cdot \nu_{x^*}^{(p)}(f,\cP).
\end{equation}

\item [$(ii)$] For $P_X$-almost sure $x^*\in \R^n$ the marginal quotient explainer $h_{S_j}^{\cP}(x^*;X,f,\vpdp)$ is well-defined and can be written as an expectation with respect to the product measure $P_j^{(h,\cP)}\otimes P_X$. Specifically,
\begin{equation}\label{eq::h^P_as_expectation}
    h_{S_j}^{\cP}(x^*;X,f,\vpdp)=\int \Delta_j(A,x;x^*,f)\,[P_j^{(h,\cP)}\otimes P_X](dA,dx)
\end{equation}

Consequently, taking random samples $(\mathcal{A}^{(k)},X^{(k)})\sim P_j^{(h,\cP)}\otimes P_X$, $k \in \{1,2,\dots,K\}$, we have
    \begin{equation}\label{eq::h^P_approx}
        \frac{1}{K}\sum_{k=1}^K \Delta_j(\mathcal{A}^{(k)},X^{(k)};x^*,f) = h_{S_j}^{\cP}(x^*;X,f,\vpdp) + \mathcal{E}_K^{\cP},
    \end{equation}
    where $\mathcal{E}_K^{\cP}\to 0$ in probability as $K\to \infty$. In addition, if $f \in [f] \in L^2(\tilde{P}_X)$, then $\mathcal{E}_K^{\cP} \to 0$ in $L^2(\P)$ with rate $O(K^{-1/2})$.
\end{itemize}
\end{theorem}

\begin{proof}
    The proof is similar to that of Theorem \ref{thm::gamevalue_as_expectation}. Specifically, the measurability of $(A,x) \mapsto \Delta_j(A,x;x^*,f)$ for $P_X$-almost sure $x^*$ follows from Lemma \ref{lmm::game_meas_marg} while the bound \eqref{Delta_bound_quot} follows from the definition of the product measure $P_j^{(h,\cP)} \otimes P_X$ and that of $\nu_{x^*}^{(p)}(f,\cP)$. The representation \eqref{eq::h^P_as_expectation} follows from the definition of the game value $h$ and that of $P_j^{(h,\cP)} \otimes P_X$. The equation \eqref{eq::h^P_approx} is a consequence of the weak law of large numbers. Finally, if $f \in [f] \in L^2(\tilde{P}_X)$, then for $P_X$-almost sure $x^*$ the error term $\EE^{\cP}_K$ is bounded in $L^2(\PP)$ by $[K^{-1} \cdot Var(\Delta_j(\mathcal{A}^{(1)},X^{(1)};x^*,f))]^{\frac{1}{2}}$ with $Var(\Delta_j(\mathcal{A}^{(1)},X^{(1)};x^*,f))$ bounded by $4 \cdot \bar{w}_{j,M} \cdot \nu_{x^*}^{(2)}(f,\cP)$.
\end{proof}

\begin{remark}\rm
The assumption of Theorem \ref{thm::quotientexplainer_as_expectation}, that the model $f \in [f] \in L^p(\tilde{P}_X)$, can be relaxed by requiring that $f \in [f] \in L^p(\tilde{P}_{X,\cP})$.
\end{remark}

\begin{remark}\rm
    If the partition $\cP=\{S_1,\dots,S_m\}$ is based on dependencies, then the marginal and conditional quotient games coincide when group predictors $X_{S_1},X_{S_2},\dots,X_{S_m}$ are independent. This means that one can use \eqref{eq::h^P_approx} to approximate $h_{S_j}^{\cP}(x;X,f,\vce)$ and obtain group explanations that take into account both the model structure and joint distribution of the predictors.
\end{remark}

The associated algorithm with the above theorem is similar to Algorithm \ref{algo_lingamevalue}. The difference between the two is the perspective on the players. In Algorithm \ref{algo_lingamevalue} the players are the predictors, while in Algorithm \ref{algo_quotgamevalue} the players are predictor groups. This is why the latter algorithm has the partition $\cP$ as an extra input requirement.

\begin{algorithm}
    \SetAlgoLined 
    \KwIn{Observation $x^*\in \R^n$, model $f$, partition $\cP=\{S_1,\dots,S_m\}$, dataset $\bar{D}_X=\{x^{(k)}\}_{k=1}^K$.
    }
    \KwOut{MC estimate $\hat{h}^{\cP}$ of $h^{\cP}(x^*;X,f,\vpdp)$.
    }
    \BlankLine
    \For{$j$ in $\{1,\dots,m\}$}
    {
        $\hat{h}_j^{\cP} := 0$;\\
        \For{$k$ in $\{1,\dots,K=|\bar{D}_X|\}$}
        {
            Select the observation $x^{(k)}$ from $\bar{D}_X$;\\
            Draw a random coalition $A^{(k)}$ from the distribution $P_j^{(h,\cP)}$;\\
            $\hat{h}_j^{\cP} := \hat{h}_j^{\cP} + \Delta_j(A^{(k)},x^{(k)};x^*,f)$;
        }
        $\hat{h}_j^{\cP} := \hat{h}_j^{\cP}/K$;
    }
    \Return{$\hat{h}^{\cP} = \big(\hat{h}_1^{\cP},\dots,\hat{h}_m^{\cP}\big)$;}
\caption{Monte Carlo sampling for quotient game values}\label{algo_quotgamevalue}\label{algo_quot_game_marg}
\end{algorithm}

\begin{remark}\rm\label{rm::emp_marg_quotgameval}
As before, to approximate the empirical marginal quotient game value $h[M,\hat{v}^{\ME,\cP}(\cdot;x^*,\bar{D}_X,f)]$ one needs to adjust Algorithm \ref{algo_quot_game_marg} as follows. The input to the algorithm must contain the number $\tilde{K}$ of MC iterations, specified independently of $|\bar{D}_X|$. Line 3 must contain the for-loop over $k \in \{1,\dots,\tilde{K}\}$ and line 4 must be replaced with ``Draw a random observation $x^{(k)}$ from the empirical distribution $P_{\bar{D}_X}$''. 

The value produced by the adjusted Algorithm \ref{algo_quot_game_marg} is the estimate of $h[M,\hat{v}^{\ME,\cP}(\cdot;x^*,\bar{D}_X,f)]$ with a mean squared error of $O(\tilde{K}^{-1})$. Finally, as in Lemma \ref{lmm::error_emp_marg_estimator}, it can be shown that the difference between the empirical marginal quotient value and the marginal quotient value in $L^2(\PP)$ is bounded by $O(|\bar{D}_X|^{-1/2})$.
\end{remark}

\subsection{Sampling for marginal coalitional values}

Quotient game explainers have several advantages over linear game values, especially when the partition $\cP$ is based on dependencies. However, the drawback is when the user is interested in explaining the contribution of each single predictor to the model output, while also taking into account the groups. This information can only be provided by coalitional explainers  as defined in Definition \ref{def::coalvalue}.

The main question we want to answer in this section is how to set up an MC sampling algorithm for a coalitional value $g[N,v,\cP]$. If one can write the element $g_i[N,v,\cP]$ in the form \eqref{lingamevalue} then Algorithm \ref{algo_lingamevalue} can be used to carry out the approximation. Given the form of some notable coalitional values such as the Owen and Banzhaf-Owen values, we assume the following form for $g_i[N,v,\cP]$:
\begin{equation}\label{coalvalform}
    g_i[N,v,\cP] = \sum_{A\subseteq M\setminus \{j\}}\sum_{T\subseteq S_j\setminus \{i\}} w_j^{(1)}(A,M) w_i^{(2)}(T,S_j) \left[ v(Q_A \cup T \cup \{i\}) - v(Q_A \cup T) \right], \quad i\in S_j,
\end{equation}
where $Q_A=\cup_{\alpha \in A}S_{\alpha}$. As with the previous subsections, we assume that $\{w^{(1)}_j(A,M)\}_{A\subseteq M\setminus \{j\}}$ and $\{w^{(2)}_i(T,S_j)\}_{T\subseteq S_j \setminus \{i\}}$ are probabilities. We again state this by saying that these sets of coefficients satisfy \eqref{hyp::coeff_nonneg}-\eqref{hyp::coeff_sumtoone}. 

Equation \eqref{coalvalform} can be written in the form of \eqref{lingamevalue}. However, this will present certain practical difficulties for the sampling procedure. For example, observe that \eqref{coalvalform} can be written as
\[
    g_i[N,v,\cP]=\sum_{S \subseteq N\setminus \{i\}} W_i(S,N,\cP)\left[ v(S \cup \{i\}) - v(S) \right],
\]
where for $i \in S_j$ we have
\[
    W_i(S,N,\cP) = 
\left\{ 
\begin{aligned}
&w^{(1)}_j(A,M)\cdot w^{(2)}_i(T,S_j), \text{if $S=\cup_{\alpha \in A} S_{\alpha} \cup T,\ A \subseteq M \setminus \{j\},\ T \subseteq S_j\setminus\{i\}$}\\
&0, \text{otherwise}.\\
\end{aligned}
\right.
\]
Thus, sampling based on the coefficients $W_i(S,N,\cP)$, $S \subseteq N$ may not be a trivial task as many coalitions $S \subseteq N$ have zero probability. For this reason, we adjust the sampling procedure to take into account the coalitional structure of \eqref{coalvalform}; see the discussion on the two-step property of coalitional values in \cite[\S 5.4.1]{grouppaper} Specifically, given a game $v$ on $N$, a partition $\cP=\{S_1,\dots,S_m\}$ and $i \in S_j$, the formulation \eqref{coalvalform} motivates the following two-step sampling procedure. One first samples a subset of $A \subseteq M$ with probability defined by the coefficients $w^{(1)}_j(A,M)$, which induces the union $Q_A=\cup_{\alpha \in A} S_{\alpha}$ of the elements of $\cP$. Then, one samples a set $T \subseteq S_j \setminus \{i\}$ with probability defined by the coefficients $w^{(2)}_i(T,S_j)$. Then forming the set $S=Q_A \cup T$ and evaluating the quantity $v(S \cup \{i\})-v(S)$ yields a sample of a random variable whose expectation is given by \eqref{coalvalform}. More formally, we have the following.

\begin{definition}\label{def::P_S_j}
    Let $g[N,v,\cP]$ be as in \eqref{coalvalform} and suppose $\{w^{(1)}_j(A,M)\}_{A\subseteq M\setminus \{j\}}$ and $\{w^{(2)}_i(T,S_j)\}_{T\subseteq S_j \setminus \{i\}}$ satisfy \eqref{hyp::coeff_nonneg}-\eqref{hyp::coeff_sumtoone}. For each $j \in M$, we define $P^{(g,\cP)}_j$ to be the probability measure satisfying $P_j^{(g,\cP)}(A)=w^{(1)}_j(A,M)$, $A\subseteq M \setminus \{j\}$, and for each $i \in S_j$, $S_j \in \cP$, we define the probability measure $P_i^{(g,S_j)}$ such that $P_i^{(g,S_j)}(T) = w^{(2)}_i(T,S_j)$, $T \subseteq S_j \setminus \{i\}$.
\end{definition}

First, we state an auxiliary lemma that represents a coalitional game value as an expectation of the random variable defined on the set of coalitions.
\begin{lemma}\label{coal_val_expect_prelim}
Let $g[N,v,\cP]$ be as in \eqref{coalvalform}, $\cP=\{S_1,\dots,S_m\}$, and suppose $i \in S_j$. Then
\[
g_i[N,v,\cP] = \int \big(v( \cup_{\alpha \in A}S_{\alpha} \cup T \cup \{i\} )-v( \cup_{\alpha \in A}S_{\alpha} \cup T )\big) [P_j^{(g,\cP)} \otimes P_i^{(g,S_j)} ](dA,dT).
\]
\end{lemma}

The above lemma states that, given a (not necessarily cooperative) game $v$, the coalitional value $g_i[N,v,\cP]$ can be viewed as an expected value with respect to a product of two probability measures. For the marginal game the above result must be adjusted to include dependence on predictors. Specifically, we can write the marginal coalitional value $g_i[N,\vpdp(\cdot;x^*,X,f),\cP]$ as an expected value with respect to a product of three probability measures as shown below.

\begin{definition}
Let $f\in [f] \in L^p(\tilde{P}_X)$, $U^{(f)}_p \in \B(\R^n)$ be a set of $P_X$-probability $1$ as in Lemma \ref{lmm::game_meas_marg}, and $\cP=\{S_1,\dots,S_m\}$ the partition of $N$. Let $x^* \in U_p^{(f)}$. For each $S_j \in \cP$ we set
\begin{equation}\label{param_marg_norm_coal}
\nu_{x^*}^{(p)}(f,\cP,S_j) := \sum_{A \subseteq M \setminus \{j\}, Q_A=\cup_{\alpha \in A} S_{\alpha}} \sum_{T \subseteq S_j } \int |f(x^*_{Q_A \cup T},x_{-(Q_A \cup T)})|^p P_{X_{-(Q_A \cup T)}}(dx_{-(Q_A \cup T)}) < \infty.
\end{equation}
\end{definition}

\begin{theorem}\label{thm::coalvalue_as_expectation}
Let $f \in [f] \in L^p(\tilde{P}_X)$. Let $\cP$ be a partition of $N$, $g[N,v,\cP]$ a coalitional value of the form \eqref{coalvalform}, and suppose that the coefficients $w^{(1)}_j(A,M)$ and $w^{(2)}_i(T,S_j)$ of $g_i$, $i\in N$ and $j\in M$, satisfy \eqref{hyp::coeff_nonneg}-\eqref{hyp::coeff_sumtoone}. 

\begin{itemize}
   \item [$(i)$] For $P_X$-almost sure $x^* \in \R^n$, for each $i \in S_j$, $S_j \in \cP$,  the map
\begin{equation}\label{Delta_coal_def}
\begin{aligned}
(A,T,x) \mapsto & \Delta_i(A,T,x;x^*,f)\\
&:=f(x^*_{(\cup_{\alpha\in A}S_{\alpha})\cup T \cup \{i\}},x_{-(\cup_{\alpha\in A}S_{\alpha})\cup T \cup \{i\}})-f(x^*_{(\cup_{\alpha\in A} S_{\alpha})\cup T},x_{-(\cup_{ \alpha \in A} S_{\alpha})\cup T}),
\end{aligned}
\end{equation}
with $A \subseteq M \setminus \{i\}$ and $T \subseteq S_j \setminus \{i\}$, is $P_j^{(g,\cP)} \otimes P_i^{(g,S_j)} \otimes  P_X $-measurable and $p$-power integrable satisfying the bound
\begin{equation}\label{Delta_bound_coal}
\begin{aligned}
& \int |\Delta_i(A,T,x;x^*,f)|^p  \, [  P_j^{(g,\cP)} \otimes P_i^{(g,S_j)} \otimes P_X](dA,dT,dx) \\
& \quad \leq  2^p \cdot \nu_{x^*}^{(p)}(f,\cP,S_j)  \max_{A \subset M \setminus \{j\}, T \subseteq S_j \setminus\{i\}} \big(w_j^{(1)}(A,M) w_i^{(2)}(T,S_j)\big).
\end{aligned}
\end{equation}

\item [$(ii)$] For $P_X$-almost sure $x^*\in \R^n$ the marginal coalitional value $g_i[N,\vpdp(\cdot; x^*,X,f),\cP]$ is well-defined for each $i \in S_j$, $S_j \in \cP$, and can be written as an expectation with respect to the product probability measure $P_j^{(g,\cP)} \otimes P_i^{(g,S_j)} \otimes P_X$. Specifically,
\begin{equation}\label{eq::g_as_expectation}
    g_i[N,\vpdp(\cdot; x^*,X,f),\cP]=\int \Delta_j(A,T,x;x^*,f)\,[P_j^{(g,\cP)} \otimes P_i^{(g,S_j)} \otimes P_X](dA,dT,dx)
\end{equation}

Consequently, taking random samples $(\mathcal{A}^{(k)}, \mathcal{T}^{(k)},X^{(k)})\sim P_j^{(g,\cP)}\otimes P_i^{(g,S_j)}\otimes P_X$, $k \in \{1,2,\dots,K\}$, we have
\begin{equation}\label{eq::g_approx}
    \frac{1}{K}\sum_{k=1}^K \Delta_i(\mathcal{A}^{(k)},\mathcal{T}^{(k)},X^{(k)};x^*,f) = g_i[N,\vpdp(\cdot; x^*,X,f),\cP] + \mathcal{E}_K^{g_i},
\end{equation}
where $\mathcal{E}^{g_i}_K \to 0$ in probability as $K\to \infty$. In addition, if $f \in [f] \in L^2(\tilde{P}_X)$, then $\mathcal{E}_K^{g_i} \to 0$ in $L^2(\P)$ with rate $O(K^{-1/2})$.
\end{itemize}
\end{theorem}

\begin{proof}
The proof is similar to that of Theorem \ref{thm::gamevalue_as_expectation}. The measurability of $(A,T,x) \mapsto \Delta_j(A,T,x;x^*,f)$ for $P_X$-almost sure $x^*$ follows from Lemma \ref{lmm::game_meas_marg} while the bound \eqref{Delta_bound_coal} follows from the definition \eqref{coalvalform} of the coalitional value $g$, the definition of the product measure $P_j^{(g,\cP)} \otimes P_i^{(g,S_j)} \otimes P_X$, and that of $\nu_{x^*}^{(p)}(f,\cP,S_j)$. 

The representation \eqref{eq::h^P_as_expectation} of the marginal coalitional value follows from the definition  \eqref{coalvalform} of  $g$ and that of $P_j^{(g,\cP)} \otimes P_i^{(g,S_j)} \otimes P_X$, while the equation \eqref{eq::g_approx} is a consequence of the weak law of large numbers. Finally, if $f \in [f] \in L^2(\tilde{P}_X)$, then for $P_X$-almost sure $x^*$ the error term $\EE^{g_i}_K$ is bounded in $L^2(\PP)$ by $[K^{-1} \cdot Var(\Delta_j(\mathcal{A}^{(1)},\mathcal{T}^{(1)},X^{(1)};x^*,f))]^{\frac{1}{2}}$ with the variance of $\Delta_j(\mathcal{A}^{(1)},\mathcal{T}^{(1)},X^{(1)};x^*,f)$ bounded by the last term in \eqref{Delta_bound_coal} for $p=2$.
\end{proof}

Equation \eqref{eq::g_approx} instructs us on how to perform the MC sampling for a coalitional value. Due to the two-step property, Algorithm \ref{algo_coalval} below is simply an augmentation of Algorithm \ref{algo_quotgamevalue}.

\begin{algorithm}
    \SetAlgoLined 
    \KwIn{Observation $x^*\in \R^n$, model $f$, partition $\cP=\{S_1,\dots,S_m\}$, dataset $\bar{D}_X=\{x^{(k)}\}_{k=1}^K$.
    }
    \KwOut{MC estimate $\hat{g}$ of $g[N,\vpdp,\cP]$.
    }
    \BlankLine
    \For{$j$ in $\{1,\dots,m\}$}
    {
        \For{$i$ in $S_j$}
        {
            $\hat{g}_i := 0$;\\
            \For{$k$ in $\{1,\dots,K=|\bar{D}_X|\}$}
            {
                Select the observation $x^{(k)}$ from $\bar{D}_X$;\\
                Draw a random coalition $A^{(k)}$ from the distribution $P_j^{(g,\cP)}$;\\
                Draw a random coalition $T^{(k)}$ from the distribution $P_i^{(S_j)}$;\\
                $\hat{g}_i := \hat{g}_i + \Delta_i(A^{(k)},T^{(k)},x^{(k)};x^*,f)$;
            }
            $\hat{g}_i := \hat{g}_i/K$;
        }
    }
    \Return{$\hat{g} = (\hat{g}_1,\dots,\hat{g}_n)$;}
\caption{Monte Carlo sampling for coalitional values}
\label{algo_coalval}
\end{algorithm}

\begin{remark}\rm\label{rm::emp_marg_coalval}
To estimate $g[N,\hat{v}^{\ME}(\cdot;x^*,\bar{D}_X,f),\cP]$ one needs to adjust Algorithm \ref{algo_coalval} as follows. The input to the algorithm must contain the number $\tilde{K}$ of MC iterations, specified independently of $|\bar{D}_X|$. Line 4 must contain the for-loop over $k \in \{1,\dots,\tilde{K}\}$ and line 5 must be replaced with ``Draw a random observation $x^{(k)}$ from the empirical distribution $P_{\bar{D}_X}$''. 

The value produced by the adjusted Algorithm \ref{algo_coalval} is the estimate of $g[N,\hat{v}^{\ME}(\cdot;x^*,\bar{D}_X,f),\cP]$ with a mean squared error of $O(\tilde{K}^{-1})$. Finally, as in Lemma \ref{lmm::error_emp_marg_estimator}, it can be shown that the difference between the empirical marginal coalitional  value and the marginal coalitional value in $L^2(\PP)$ is bounded by $O(|\bar{D}_X|^{-1/2})$.
\end{remark}

\subsection{Sampling for two-step Shapley values}

Continuing the discussion on coalitional values, there are several of them that are not of the form \eqref{coalvalform}. One in particular that is of note is called two-step Shapley, defined in \cite{Kamijo2009} and given by

\begin{equation}\label{eq::twostepshap}
    TSh_i[N,v,\cP] = \varphi_i[S_j,v]+\frac{1}{|S_j|}\big(\varphi_j[M,v^{\cP}]-v(S_j)\big), \quad i\in S_j.
\end{equation}
Observe that two-step Shapley consists of three terms: the Shapley value of the player $i$ when the game $v$ is restricted to the group $S_j$, the quotient game Shapley value of the group $S_j$, and the game itself evaluated at $S_j$. Furthermore, if $|S_j|=1$, \eqref{eq::twostepshap} simply reduces to $\varphi_j[M,v^{\cP}]$.

Although Algorithm \ref{algo_coalval} cannot be directly applied to approximate $TSh_i[N,\vpdp,\cP]$, each individual term of \eqref{eq::twostepshap} can be evaluated via sampling based on Algorithm \ref{algo_lingamevalue} and Algorithm \ref{algo_quotgamevalue}. Given an observation $x^*\in \R^n$, a model $f$, and a background dataset $\bar{D}_X$, we have

\begin{itemize}
    \item[(i)] The value $\varphi_j[M,v^{\ME,\cP}(\cdot; x^*,X,f)]$ can be approximated via Algorithm \ref{algo_quot_game_marg}, where
    \[
        P_j^{(\varphi,\cP)}(A):=\frac{|A|!(m-|A|-1)!}{m!}, \quad A\subseteq M\setminus \{j\}.
    \]
    \item[(ii)] $\vpdp(S_j;x^*,X,f)$ can be approximated by averaging the values of $f(x_{S_j}^*,\cdot)$ over $\bar{D}_X$.
    \item[(iii)] The value $\varphi_i[S_j,\vpdp(\cdot; x^*,X,f)]$ can be approximated using Algorithm \ref{algo_lingamevalue} by adjusting it to accommodate the fact that the game is restricted to $S_j$. Specifically,  Line 5 has to provide the sampling of coalitions from $S_j \setminus \{i\}$ using the probability distribution
        \[
        P_i^{(\varphi,S_j)}(S) := \frac{|S|!(|S_j|-|S|-1)!}{|S_j|!}, \quad S\subseteq S_j\setminus \{i\}.
    \]
\end{itemize}

Combining the above steps leads to Algorithm \ref{algo_twostep} for $TSh_i[N,\vpdp,\cP]$.

\begin{algorithm}
    \SetAlgoLined 
    \KwIn{Observation $x^*\in \R^n$, model $f$, partition $\cP=\{S_1,\dots,S_m\}$, dataset $\bar{D}_X=\{x^{(k)}\}_{k=1}^K$.
    }
    \KwOut{MC estimate $\hat{TSh}$ of $TSh[N,\vpdp,\cP]$.
    }
    \BlankLine
    \For{$j$ in $\{1,\dots,m\}$}
    {
        \For{$i$ in $S_j$}
        {
            \If {$|S_j|==1$}{
                Apply the loop from Algorithm \ref{algo_quot_game_marg} and obtain $\hat{TSh}_i$;
            }
            \Else{
                $\hat{TSh}_i := 0$;\\
                \For{$k$ in $\{1,\dots,K=|\bar{D}_X|\}$}
                {
                    Select the observation $x^{(k)}$ from $\bar{D}_X$;\\
                    $f_j^{(k)}:=f(x_{S_j}^*,x_{-S_j}^{(k)})$\\
                    Draw a random coalition $A^{(k)}$ from the distribution $P_j^{(\varphi,\cP)}$;\\

                    $\Delta_j(A^{(k)},x^{(k)};x^*,f):=f(x_{\cup_{\alpha\in A^{(k)}\cup \{j\}}S_{\alpha}}^*,x_{-\cup_{\alpha\in A^{(k)}\cup \{j\}}S_{\alpha}}^{(k)})-f(x_{\cup_{\alpha\in A^{(k)}}S_{\alpha}}^*,x_{-\cup_{\alpha\in A^{(k)}}S_{\alpha}}^{(k)})$;\\
                    Draw a random coalition $S^{(k)}$ from the distribution $P_i^{(\varphi,S_j)}$;\\
                    
                    $\Delta_i(S^{(k)},x^{(k)};x^*,f):=f(x_{S^{(k)}\cup \{i\}}^*,x_{-S^{(k)}\cup \{i\}}^{(k)})-f(x_{S^{(k)}}^*,x_{-S^{(k)}})$;\\
                    \BlankLine
                    $\hat{TSh}_i := \hat{TSh}_i + \Delta_i(S^{(k)},x^{(k)};x^*,f)+\frac{1}{|S_j|}(\Delta_j(A^{(k)},x^{(k)};x^*,f)-f_j^{(k)})$;
                }
                $\hat{TSh}_i := \hat{TSh}_i/K$;
            }
        }
    }
    \Return{$\hat{TSh} = (\hat{TSh}_1,\dots,\hat{TSh}_n)$;}
\caption{Monte Carlo sampling for two-step Shapley values}
\label{algo_twostep}
\end{algorithm}

\begin{remark}\rm\label{rm::emp_marg_twostep}
To estimate $TSh[N,\hat{v}^{\ME}(\cdot;x^*,\bar{D}_X,f),\cP]$, one needs to adjust Algorithm \ref{algo_twostep} as follows. The input to the algorithm must contain the number $\tilde{K}$ of MC iterations, specified independently of $|\bar{D}_X|$. Line 8 must contain the for-loop over $k \in \{1,\dots,\tilde{K}\}$ and line 9 must be replaced with ``Draw a random observation $x^{(k)}$ from the empirical distribution $P_{\bar{D}_X}$''. 

The value produced by the adjusted Algorithm \ref{algo_twostep} is the estimate of $TSh[N,\hat{v}^{\ME}(\cdot;x^*,\bar{D}_X,f),\cP]$ with a mean squared error of $O(\tilde{K}^{-1})$. Finally, as in Lemma \ref{lmm::error_emp_marg_estimator}, it can be shown that the difference between the empirical marginal two-step Shapley value and the marginal two-step Shapley value in $L^2(\PP)$ is bounded by $O(|\bar{D}_X|^{-1/2})$.
\end{remark}

\section{Accelerated Monte Carlo sampling with precomputations}\label{sec::altMC_algos}

The algorithms presented in Section \ref{sec::MC_theory} on estimating explainers all contain parts that require real-time sampling (of coalitions $A\subseteq M$ and $S\subseteq N$) and repeated calls to the model $f$. One way of reducing the computation time of these algorithms is to separate out the sampling part and have it be executed as a precomputation step. The result of the precomputation can then be passed as an input to the computation step, which would execute the averaging of the MC iterates preferably in a manner that minimizes the number of model calls.

The idea of incorporating a precomputation step to generate and store samples for the MC iterations, while it is sound, may not always be feasible. To understand why, consider Algorithm \ref{algo_game_marg}. Observe that for each feature $i\in\{1,\dots,n\}$, $K=|\bar{D}_X|$ samples of coalitions must be generated (where each coalition is stored as a vector of zeroes and ones of size $n$), which means that the total memory required to store all $n$ matrices of coalition samples is $O(n^2 \cdot K)$ bytes. This can become exceedingly large if the number of features $n$ and MC samples $K$ are large.

To address this issue, we propose a methodology that allows us to reduce the number of stored matrices from $n$ to $1$, which will decrease the runtime of the precomputation and the memory required for storage. In order to accomplish this, one has to reformulate the linear game value \eqref{lingamevalue} so that the summation does not depend on the index $i$. In other words, we shift the sampling procedure from one based on $P_i^{(h)}$ to a probability distribution independent of $i$, which in turn allows us to generate only one set of coalition samples and use that to generate explanations for every feature $i \in N$.

\medskip
\noindent {\bf Accelerated-MC sampling for Shapley values.}

We showcase  the algorithm for the Shapley value \eqref{shapform} and then explain how these calculations can be extended to any linear game value; see Remark \ref{fastmc_extension}.

\begin{lemma}
    Define $P^{(\varphi)}(\{S\}) = \frac{1}{n}\frac{s!(n-s)!}{n!}$ on the space $\Omega = \{S:S\subsetneq N\}$, $s = |S|$, $n=|N|$. Then $P^{(\varphi)}(\Omega) = 1$ and for any game $(N,v)$
    \begin{equation}\label{altshapform}
        \varphi_i[N,v] = \int \left[\left(\frac{|N|}{|N|-|S|}\right) \left(v(S \cup \{i\}) - v( S )\right) \right]\,P^{(\varphi)}(dS), \quad i\in N.
    \end{equation}
\end{lemma}

\begin{proof}
    $P^{(\varphi)}$ defines a probability measure on $\Omega$. For any $S\in \Omega$, $P^{(\varphi)}(\{S\})$ is nonnegative and

    \begin{equation*}
        \begin{aligned}
        P^{(\varphi)}(\Omega) &= \sum_{S\in \Omega}P^{(\varphi)}(\{S\}) = \sum_{S\subsetneq N} \frac{1}{n}\frac{s!(n-s)!}{n!} = \sum_{S\subsetneq N} \frac{1}{n}\binom{n}{s}^{-1} = \sum_{\gamma=0}^{n-1}\sum_{S\subsetneq N,|S|=\gamma} \frac{1}{n}\binom{n}{\gamma}^{-1}\\
        &= \sum_{\gamma=0}^{n-1} \frac{1}{n}\binom{n}{\gamma}^{-1}\hspace{-10 pt} \sum_{S\subsetneq N,|S|=\gamma} 1 = \sum_{\gamma=0}^{n-1} \frac{1}{n}\binom{n}{\gamma}^{-1}\binom{n}{\gamma} = \sum_{\gamma=0}^{n-1} \frac{1}{n} = 1.
        \end{aligned}
    \end{equation*}
Furthermore, the Shapley value can be rewritten as an integral with respect to $P^{(\varphi)}$ as follows.
    \begin{equation*}
        \begin{aligned}
            \varphi_i[N,v] &= \sum_{S \subseteq N \backslash\{i\}} \frac{s!(n-s-1)!}{n!} [ v(S \cup \{i\}) - v( S ) ] = \sum_{S \subsetneq N} \frac{1}{n}\frac{s!(n-s)!}{n!} \left[\left(\frac{n}{n-s}\right) \left(v(S \cup \{i\}) - v( S )\right) \right]\\
            &= \int \left[\left( \frac{|N|}{|N|-|S|}\right) \left(v(S \cup \{i\}) - v( S )\right) \right]\,P^{(\varphi)}(dS).
        \end{aligned}
    \end{equation*}
\end{proof}

Notice that in \eqref{altshapform} the only dependence on the index $i$ is in the term $v(S\cup i)$. The coalitions $S$ are now independent of $i$ since we are considering proper subsets of $N$ and not subsets of $N\setminus \{i\}$. Practically, this has the implication that one can sample coalitions and reuse them to estimate $\varphi_i[N,v]$ for all $i\in N$. Algorithms \ref{algo_pre_fastmc_shapley} and \ref{algo_comp_fastmc_shapley} describe the precomputation and computation steps, respectively, where the former generates and stores the sampled coalitions and the coefficients $\frac{n}{n-s}$, and the latter simply loads that stored information and produces the MC estimate of $\varphi[N,\vpdp]$.

\begin{algorithm}
    \SetAlgoLined 
    \KwIn{Dataset $\bar{D}_X=\{x^{(k)}\}_{k=1}^K$.}
    \KwOut{Matrix $C$ and vector $w$.}
    \BlankLine
    Initialize a zero matrix $C$ of size $K\times n$;\\
    Initialize a zero vector $w$ of length $K$;\\
    \For{$k$ in $\{1,\dots,K\}$} 
    {
        Pick $r_k$ from the set $\in \{0,1,\dots,n-1\}$ uniformly at random;\\
        Randomly generate a logical array $c_k$ of length $n$ such that $\sum_{\ell=1}^n c_k[\ell] = r_k$;\\
        Set $C[k,\cdot] = c_k$;\\
        Set $w[k] = \frac{n}{n-r_k}$;\\
    }
    \Return{$C$ and $w$;}
\caption{Precomputation step of Accelerated-MC sampling for marginal Shapley values}
\label{algo_pre_fastmc_shapley}
\end{algorithm}

\begin{algorithm}
    \SetAlgoLined 
    \KwIn{Observation $x^*\in \R^n$, model $f$, dataset $\bar{D}_X=\{x^{(k)}\}_{k=1}^K$, matrix $C$ and vector $w$ from Algorithm \ref{algo_pre_fastmc_shapley}.
    }
    \KwOut{MC estimate $\hat{\varphi}=\{\hat{\varphi}\}_{i=1}^n$ of $\varphi[N,\vpdp]$.}
    \BlankLine
    Initialize a zero vector $\hat{\varphi}$ of length $n$;\\
    Set $X_{synth} = x^* \cdot C + \bar{D}_X\cdot (1-C)$;\\
    \Comment*[h]{$/^*$ Both multiplications in line 2 are component-wise. In the former, every row of $C$ is multiplied component-wise with $x^*$, creating a matrix with the same size as $C$ $^*/$}\\
    Set $X_{synth,copy} = copy(X_{synth})$;\\
    Evaluate $f_{synth} = f(X_{synth})$;\\
    \For{$i$ in $\{1,\dots,n\}$}
    {
        Set $X_{synth,copy}[\cdot,i] = x^*[i]$;\\
        Evaluate $f_{synth,i} = f(X_{synth,copy})$;\\
        Reset $X_{synth,copy}[\cdot,i] = X_{synth}[\cdot,i]$;\\
        \For{$k$ in $\{1,\dots,K\}$ such that $C[k,i]=0$}
        {
            Evaluate $\hat{\varphi}_i = \hat{\varphi}_i + w[k](f_{synth,i}[k] - f_{synth}[k])$;\\
        }
        Set $\hat{\varphi}_i = \hat{\varphi}_i/K$;\\
    }
    \Return{$\hat{\varphi} = (\hat{\varphi}_1,\dots,\hat{\varphi}_n)$;}
\caption{Computation step of Accelerated-MC sampling for marginal Shapley values}
\label{algo_comp_fastmc_shapley}
\end{algorithm}

\begin{remark}\rm\label{fastmc_extension}
To extend the Accelerated-MC sampling to any linear game value of the form \eqref{lingamevalue}, suppose that the coefficient $w_i(S,N)$ is renormalized to $w'_i(S,N) = c(S,N)w_i(S,N)$ so that $\sum_{S\subsetneq N}w'_i(S,N) = 1$. Then the term $v(S\cup \{i\}) - v(S)$ is multiplied by $c(S,N)^{-1}$. To generalize Algorithm \ref{algo_pre_fastmc_shapley} to any linear game value, notice that lines $4-5$ define the distribution of 1's in the logical array $c_k$ based on the probability $\frac{1}{n}\frac{r_k!(n-r_k)!}{n!}$ (the coefficient of the reformulated Shapley value \eqref{altshapform}), where $r_k$ specifies the number of 1's in the array. Thus, in the general case, these lines should be changed so that the distribution of 1's in $c_k$ is based on the probability $w'_i(S,N)$. Furthermore, line 7 should change to $w[k] = c(S,N)^{-1}$. No modifications are required in Algorithm \ref{algo_comp_fastmc_shapley}.
\end{remark}

\begin{remark}\rm
Note that $v(S \cup \{i\})-v(S)=0$ if $i \in S$. Thus, the difference of functions in line $10$ of Algorithm \ref{algo_comp_fastmc_shapley} has to be computed only for  those samples of coalitions which do not contain $i$.
\end{remark}

\begin{remark}\rm
    As with the previous algorithms presented in Section \ref{sec::MC_theory}, Algorithms \ref{algo_pre_fastmc_shapley} and \ref{algo_comp_fastmc_shapley} can be adjusted to estimate the Shapley values for the empirical marginal game $\hat{v}^{\ME}$ (based on the dataset $\bar{D}_X$) by sampling with replacement from $\bar{D}_X$, creating a new dataset $\tilde{D}$. Using this set as an input to the algorithms instead of $\bar{D}_X$ yields the estimate.
\end{remark}

\begin{remark}\rm
    The benefit of having precomputed samples is that in Algorithm \ref{algo_comp_fastmc_shapley} the input to the model $f$ can be the entire matrix of samples, which minimizes the call to the function. The mutliple function calls for each sample is the main bottleneck in the algorithms presented in Section \ref{sec::MC_theory}.
\end{remark}

\medskip
\noindent {\bf Accelerated-MC sampling for Owen values.} Next, we adapt the ideas above to the Owen value \eqref{OwenandBzOw} to produce sampling algorithms with a precomputation and computation step, respectively. The calculations again can be extended to any coalitional value of the form \eqref{coalvalform}; see Remark \ref{fastmc_extension_owen}.

\begin{lemma}
    Define $P_j^{(Ow,\cP)}(\{(R,T)\}) = \frac{1}{m}\frac{r!(m-r)!}{m!}\cdot \frac{1}{s_j}\frac{t!(s_j-t)!}{s_j!}$ on $\Omega_j^{\cP} = \{(R,T):R\subsetneq M,\ T\subsetneq S_j\}$, where $S_j \in \cP = \{S_1,\dots,S_m\}$, $s_j = |S_j|$, $r=|R|$ and $t=|T|$. Then $P_j^{(Ow,\cP)}(\Omega_j^{\cP}) = 1$ and for any coalitional game $[N,v,\cP]$ and $i\in S_j$
    \begin{equation}\label{altowenform}
        Ow_i[N,v,\cP] = \int \left[\left(\frac{|M|}{|M|-|R|}\frac{|S_j|}{|S_j|-|T|}\right) \left(v(Q_R \cup T \cup \{i\}) - v( Q_R \cup T )\right) \right]\,P_j^{(Ow,\cP)}(dR,dT),
    \end{equation}
    where $Q_R = \cup_{\alpha \in R}S_{\alpha}$ and $M=\{1,2,\dots,m\}$.
\end{lemma}

\begin{proof}
    Given $j\in M$ and a partition $\cP$, $P_j^{(Ow,\cP)}$ defines a probability measure on $\Omega_j^{\cP}$. For any $(R,T)\in \Omega_j^{\cP}$, $P_j^{(Ow,\cP)}(\{(R,T)\})$ is nonnegative and

    \begin{equation*}
        \begin{aligned}
        P_j^{(Ow,\cP)}(\Omega_j^{\cP}) &= \sum_{(R,T)\in \Omega_j^{\cP}}P_j^{(Ow,\cP)}(\{R,T\}) = \sum_{R\subsetneq M}\sum_{T\subsetneq S_j} \frac{1}{m}\frac{r!(m-r)!}{m!}\cdot \frac{1}{s_j}\frac{t!(s_j-t)!}{s_j!}\\
         &= \sum_{R\subsetneq M}\sum_{T\subsetneq S_j} \frac{1}{m}\binom{m}{r}^{-1} \frac{1}{s_j}\binom{s_j}{t}^{-1} = \frac{1}{m s_j} \left(\sum_{R\subsetneq M} \binom{m}{r}^{-1}\right) \left(\sum_{T\subsetneq S_j}\binom{s_j}{t}^{-1}\right)\\
        &= \frac{1}{m s_j} \left(\sum_{\gamma=0}^{m-1}\sum_{R\subsetneq M,|R|=\gamma} \binom{m}{\gamma}^{-1}\right) \left(\sum_{\zeta=0}^{s_j-1}\sum_{T\subsetneq S_j,|T|=\zeta}\binom{s_j}{\zeta}^{-1}\right)\\
        &= \frac{1}{m s_j} \left(\sum_{\gamma=0}^{m-1}\binom{m}{\gamma}^{-1}\binom{m}{\gamma}\right) \left(\sum_{\zeta=0}^{s_j-1}\binom{s_j}{\zeta}^{-1}\binom{s_j}{\zeta}\right) = 1.
        \end{aligned}
    \end{equation*}
Furthermore, the Owen value $Ow_i[N,v,\cP]$ for $i\in S_j$ can be rewritten as an integral with respect to $P_j^{(Ow,\cP)}$ as follows.
    \begin{equation*}
        \begin{aligned}
            Ow_i[N,v,\cP] &= \sum_{R\subseteq M\setminus \{j\}}\sum_{T\subseteq S_j\setminus \{i\}} \tfrac{|R|!(|M|-|R|-1)!}{|M|!}\tfrac{|T|!(|S_j|-|T|-1)!}{|S_j|!} \left[ v(Q_R \cup T \cup \{i\}) - v(Q_R \cup T) \right]\\
            &= \sum_{R\subsetneq M}\sum_{T\subsetneq S_j} \tfrac{1}{|M|}\tfrac{|R|!(|M|-|R|)!}{|M|!}\tfrac{1}{|S_j|}\tfrac{|T|!(|S_j|-|T|)!}{|S_j|!} \left(\tfrac{|M|}{|M|-|R|}\tfrac{|S_j|}{|S_j|-|T|}\right) \left[v(Q_R \cup T \cup \{i\}) - v(Q_R \cup T)\right]\\
            &= \int \left[\left(\frac{|M|}{|M|-|R|}\frac{|S_j|}{|S_j|-|T|}\right) \left(v(Q_R \cup T \cup \{i\}) - v( Q_R \cup T )\right) \right]\,P_j^{(Ow,\cP)}(dR,dT).
        \end{aligned}
    \end{equation*}
\end{proof}

Notice that in \eqref{altowenform} the sampling of the set $T \subsetneq S_j$ depends on the index $j$. Due to the nature of coalitional values the dependence on $S_j$ cannot be fully removed from the second summation. Practically, this has the implication that one can sample coalitions of unions $Q_R$ and reuse them to estimate $Ow_i[N,v,\cP]$ for all $i\in N$. However, separate samples of coalitions $T$ must be generated for each group $S_j$, $j\in M$. Algorithms \ref{algo_pre_fastmc_owen} and \ref{algo_comp_fastmc_owen} describe the precomputation and computation steps, respectively, where the former generates and stores the sampled coalitions of unions, the sampled coalitions of predictor indices for each group, and the coefficients $\frac{m}{m-r}\cdot \frac{s_j}{s_j-t}$. The latter algorithm simply loads that stored information and produces the MC estimate of $Ow[N,\vpdp,\cP]$.

\begin{algorithm}
    \SetAlgoLined 
    \KwIn{Partition $\cP=\{S_1,\dots,S_m\}$, dataset $\bar{D}_X=\{x^{(k)}\}_{k=1}^K$.
    }
    \KwOut{Matrices $\{C_i\}_{i=1}^m$ and $W$.
    }
    \BlankLine
    Initialize a zero matrix $R$ of size $K\times m$;\\
    Initialize a zero matrix $A$ of size $m\times n$;\\
    Initialize a zero matrix $W$ of size $m\times K$;\\    
    \For{$k$ in $\{1,\dots,K\}$}
    {
        Pick $q_k$ from the set $\in \{0,1,\dots,m-1\}$ uniformly at random;\\
        Randomly generate a logical array $r_k$ of length $m$ such that $\sum_{\ell=1}^m r_k[\ell] = q_k$;\\
        Set $R[k,\cdot] = r_k$;\\
    }
    \For{$j$ in $\{1,\dots,m\}$}
    {
        Set $A[j,S_j] = 1$;\\
        Initialize a zero matrix $T_j$ of size $K\times |S_j|$;\\
        \For{$k$ in $\{1,\dots,K\}$}
        {
            Pick $d_k$ from the set $\in \{0,1,\dots,|S_j|-1\}$ uniformly at random;\\
            Randomly generate a logical array $t_k$ of length $|S_j|$ such that $\sum_{\ell=1}^{|S_j|} t_k[\ell] = d_k$;\\
            Set $T_j[k,\cdot] = t_k$;\\
        }
    }
    Set $Q = RA$ (matrix multiplication);\\
    \For{$j$ in $\{1,\dots,m\}$}
    {
        Set $C_j = Q$;\\
        Set $C_j[\cdot, S_j] = Q[\cdot, S_j] + T_j$;\\
        Set any element of $C_j$ that is greater than $1$ equal to $1$;\\
        \For{$k$ in $\{1,\dots,K\}$}
        {
            Set $W[j,k] = \frac{m}{m-\sum_{\ell=1}^m R[k,\ell]}\cdot \frac{|S_j|}{|S_j| - \sum_{\ell=1}^{|S_j|} T_j[k,\ell]}$;\\
        }
    }
    \Return{$\{C_j\}_{j=1}^m$ and $W$;}
\caption{Precomputation step of Accelerated-MC sampling for marginal Owen values}
\label{algo_pre_fastmc_owen}
\end{algorithm}

\begin{algorithm}
    \SetAlgoLined 
    \KwIn{Observation $x^*\in \R^n$, model $f$, partition $\cP=\{S_1,\dots,S_m\}$, dataset $\bar{D}_X=\{x^{(k)}\}_{k=1}^K$, matrices $\{C_j\}_{j=1}^m$ and $W$ from Algorithm \ref{algo_pre_fastmc_owen}.
    }
    \KwOut{MC estimate $\hat{Ow}=\{\hat{Ow}_i\}_{i=1}^n$ of $Ow[N,\vpdp,\cP]$.
    }
    \BlankLine
    Initialize a zero vector $\hat{Ow}$ of length $n$;\\
    \For{$j$ in $\{1,\dots,m\}$}
    {
        Set $X_{synth} = x^* \cdot C_j + \bar{D}_X\cdot (1-C_j)$;\\
        \Comment*[h]{$/^*$ Both multiplications in line 3 are component-wise. In the former, every row of $C_j$ is multiplied component-wise with $x^*$, creating a matrix with the same size as $C_j$ $^*/$}\\
        Set $X_{synth,copy} = copy(X_{synth})$;\\
        Evaluate $f_{synth} = f(X_{synth})$;\\
        \For{$i$ in $S_j$}
        {
            Set $X_{synth,copy}[\cdot,i] = x^*[i]$;\\
            Evaluate $f_{synth,i} = f(X_{synth,copy})$;\\
            Reset $X_{synth,copy}[\cdot,i] = X_{synth}[\cdot,i]$;\\
            \For{$k$ in $\{1,\dots,K\}$ such that $C_j[k,i]=0$}
            {
                Evaluate $\hat{Ow}_i = \hat{Ow}_i + W[j,k](f_{synth,i}[k] - f_{synth}[k])$;\\
            }
            Set $\hat{Ow}_i = \hat{Ow}_i/K$;\\
        }
    }
    \Return{$\hat{Ow} = (\hat{Ow}_1,\dots,\hat{Ow}_n)$;}
\caption{Computation step of Accelerated-MC sampling for marginal Owen values}
\label{algo_comp_fastmc_owen}
\end{algorithm}

\begin{remark}\rm\label{fastmc_extension_owen}
    To extend the Accelerated-MC sampling to any coalitional value of the form \eqref{coalvalform}, suppose that the coefficient is renormalized to $w'_i(A,M,T,S_j) = c^{(1)}(A,M)w_j^{(1)}(A,M)c^{(2)}(T,S_j)w_i^{(2)}(T,S_j)$ so that $\sum_{A\subsetneq M}c^{(1)}(A,M)w_j^{(1)}(A,M) = \sum_{T\subsetneq S_j}c^{(2)}(T,S_j)w_i^{(2)}(T,S_j) = \sum_{A\subsetneq M}\sum_{T\subsetneq S_j}w'_i(A,M,T,S_j) = 1$. Then the term $v(Q_A\cup T\cup \{i\}) - v(Q_A\cup T)$ is multiplied by $(c^{(1)}(A,M)c^{(2)}(T,S_j))^{-1}$. Similar to Remark \ref{fastmc_extension}, to generalize Algorithm \ref{algo_pre_fastmc_owen} to any coalitional value, lines $5-6$ should be changed so that the distribution of 1's in $r_k$ is based on the probability $c^{(1)}(R,M)w_j^{(1)}(R,M)$, and lines $13-14$ should be changed so that the distribution of 1's in $t_k$ is based on the probability $c^{(2)}(T,S_j)w_i^{(2)}(T,S_j)$. Furthermore, line $24$ should change to $W[j,k] = (c^{(1)}(R,M)c^{(2)}(T,S_j))^{-1}$. No modifications are required in Algorithm \ref{algo_comp_fastmc_owen}.
\end{remark}

\begin{remark}\rm
    As with the previous algorithms, Algorithms \ref{algo_pre_fastmc_owen} and \ref{algo_comp_fastmc_owen} can be adjusted to estimate the Owen value for the empirical marginal game $\hat{v}^{\ME}$ (based on the dataset $\bar{D}_X$) by sampling with replacement from $\bar{D}_X$, creating a new dataset $\tilde{D}$. Using this set as an input to the algorithms instead of $\bar{D}_X$ yields the estimate.
\end{remark}

\section{Numerical Examples}\label{sec::numerical}

In this section we present the results from the numerical experiments conducted on synthetic data. The purpose of these experiments is to numerically show the asymptotic convergence of our MC approach for three empirical marginal game values: quotient Shapley values, Owen values, and two-step Shapley values.

The reason why we estimate the empirical marginal game value and not the true marginal is two-fold. First, because in practice it is infeasible to obtain the true marginal game value to have as a ground truth. On the other hand, the empirical marginal game value can be evaluated when the background dataset is small enough (recall that the complexity for $h_i[N,\hat{v}^{\ME}]$ is $O(2^n\cdot |\bar{D}_X|)$). Second, to properly showcase the asymptotic behavior of the estimator, we need to be able to carry out as many MC iterations as we desire. The version of the algorithm that estimates the true marginal game value would bound the number of MC iterations by $|\bar{D}_X|$.

There are two subsections below. In the first subsection we explain how the algorithms for the three aforementioned game values can be applied in a practical setting, which is the approach used to carry out the experiments. The following subsection contains the synthetic data examples, presents the results and provides a discussion.

\subsection{Practical implementation of algorithms}

In each of the Algorithms \ref{algo_game_marg}-\ref{algo_twostep}, there is at least one line that dictates to randomly draw from some distribution. Specifically, when estimating the true or empirical marginal game value one has to randomly draw coalitions, and in the case of the empirical marginal game value one also has to sample observations from the background dataset $\bar{D}_X$. Random sampling from $\bar{D}_X$ is not practically difficult. However, depending on the coefficients of the game value, it may not always be easy to randomly select a coalition.

Luckily, for game values that incorporate Shapley coefficients, such as the ones we will work with in our examples, the coalition sampling can be done in a relatively straightforward manner. The key is to use random permutations of player indices. We will explain in detail how to apply this idea to Algorithm \ref{algo_game_marg} when $h[N,\vpdp]$ is $\varphi[N,\vpdp]$ and then state the necessary changes to the other algorithms.

As described in Remark \ref{rem::shapley_probs}, the probability of sampling the coalition $S\subseteq N\setminus \{i\}$ is given by $P_i^{(\varphi)}(S) = \frac{|S|!(n-|S|-1)!}{n!}$. Now let $\sigma:N\to N$ be a permutation of the indices in $N$, and suppose $i$ is the player of interest whose Shapley value $\varphi_i[N,\vpdp]$ we would like to estimate. The below steps explain how random permutations work in sampling coalitions based on the Shapley coefficients:
\begin{itemize}
    \item[(1)] Suppose $\sigma:N\to N$ is a (selected uniformly at random) permutation of the indices in $N$ such that $\sigma(\ell)=i$ for some $\ell\in N$. This yields the permutation $\sigma(N) = (\sigma(1),\sigma(2),\dots,\sigma(\ell-1),i,\sigma(\ell+1),\dots,\sigma(n))$.
    \item[(2)] Consider the set of the first $l-1$ indices in the permuted set. Set $S(\sigma,i)=\{\sigma(1),\sigma(2),\dots,\sigma(\ell-1)\}$ as the sampled coalition. Then, the probability of observing a particular set $S^* \subseteq N \setminus \{i\}$ is given by
    \[
    P_{\sigma}(S(\sigma,i)=S^*)=\frac{|S^*|!(n-|S^*|-1)!}{n!}.
    \]
\end{itemize}

Notice that the above formula turns out to be exactly $P_i^{(\varphi)}(S^*)$. Therefore, in order to sample a coalition based on $P_i^{(\varphi)}$, one simply needs to randomly permute $N$ and set the sampled coalition $S$ as the indices to the left of the player of interest. The only thing that remains is to prove the claim in step (2).

\medskip

\noindent\textit{Claim: Let $i\in N$ be the player of interest,  $S^* \subseteq N \setminus \{i\}$ and $l=|S^*|+1$. The probability of sampling $S^*$ based on the Shapley value coefficients is equal to randomly selecting a permutation $\sigma:N\to N$ such that $S^*=\{\sigma(1),\sigma(2),\dots,\sigma(\ell-1)\}$ and $\sigma(\ell)=i$.}

\begin{proof}    
    Given $S^* \subseteq N \setminus \{i\}$, we set $l=|S^*|+1$. Then
    \[
  P_{\sigma}(S(\sigma,i)=S^*)=P_{\sigma}(S(\sigma,i)=S^* | \sigma(l)= i )P_{\sigma}(\sigma(l)=i).
    \]

To obtain the probability $P_{\sigma}(S(\sigma,i)=S^* | \sigma(l)=i )$, observe that this is simply the inverse of the binomial coefficient $\binom{n-1}{\ell-1}$ since there are that many distinct choices for the indices to the left of $i$ in the permuted set $\sigma(N)$ where $\sigma(\ell)=i$. Next, since $\sigma$ is a permutation of $N$, then $P_{\sigma}(\sigma(l)=i)=\frac{1}{n}$. Putting the two together proves the claim.
\end{proof}

Considering the above, each algorithm can be adjusted as follows for practical purposes when Shapley coefficients are involved:

\medskip
\textbf{Algorithm \ref{algo_game_marg} for $\varphi[N,\vpdp]$:} Replace line 5 with ``Take a random permutation $\sigma$ of $N$. Suppose $\sigma(\ell)=i$ for some $\ell \in N$. Then set $S^{(k)}=\{\sigma(1),\dots,\sigma(\ell-1)\}$''.

\medskip
\textbf{Algorithm \ref{algo_quot_game_marg} for $\varphi[M,v^{\ME,\cP}]$:} Replace line 5 with ``Take a random permutation $\sigma$ of $M$. Suppose $\sigma(\ell)=j$ for some $\ell \in M$. Then set $A^{(k)}=\{\sigma(1),\dots,\sigma(\ell-1)\}$''.

\medskip
\textbf{Algorithm \ref{algo_coalval} for $Ow[N,\vpdp,\cP]$:} Replace line 6 with ``Take a random permutation $\sigma$ of $M$. Suppose $\sigma(\ell)=j$ for some $\ell \in M$. Then set $A^{(k)}=\{\sigma(1),\dots,\sigma(\ell-1)\}$''. Next, replace line 7 with ``Take a random permutation $\tau(\{1,\dots,s\})$ where $S_j=\{i_1,\dots,i_s\}$. Suppose $i_{\tau(t)}=i$ for some $t\in \{1,\dots,s\}$. Then set $T^{(k)}=\{i_{\tau(1)},\dots,i_{\tau(t-1)}\}$''.

\medskip
\textbf{Algorithm \ref{algo_twostep}:} Replace line 11 with ``Take a random permutation $\sigma$ of $M$. Suppose $\sigma(\ell)=j$ for some $\ell \in M$. Then set $A^{(k)}=\{\sigma(1),\dots,\sigma(\ell-1)\}$''. Next, replace line 13 with ``Take a random permutation $\tau(\{1,\dots,s\})$ where $S_j=\{i_1,\dots,i_s\}$. Suppose $i_{\tau(t)}=i$ for some $t\in \{1,\dots,s\}$. Then set $S^{(k)}=\{i_{\tau(1)},\dots,i_{\tau(t-1)}\}$''.

\begin{remark}\rm
    The same adjustments apply when one applies the version of each algorithm that estimates the empirical marginal game value.
\end{remark}

\subsection{Synthetic data examples}\label{subsec::num_experiments}

Now we will present six different numerical experiments conducted using the second version of the algorithms that estimates the empirical marginal game value.

For each of the three game values mentioned in the beginning of this section (quotient Shapley, Owen, and two-step Shapley values), there are two experiments conducted. One where we showcase the error of estimation as the number of MC iterations increase, and the other is a similar experiment that also compares the different errors for varying number of predictors in the model $f$. The latter experiment is to provide numerical evidence that the number of predictors does not substantially impact the relative error.

In detail, the experiments are set up as follows. A data generating model $Y=f(X)$ is constructed, whose specifics will be provided once each experiment is discussed, and the background dataset $\bar{D}_X$ is sampled from the distribution of $X$ and is of size $100$. The number $K$ of MC iterations will vary in the set $\{2^r\}_{r=9}^{14}$ to depict the asymptotic behavior of the error estimation. As for the error itself, we evaluate the Mean Integrated Squared Error (MISE) between the empirical marginal game value that has been directly computed on all observations in $\bar{D}_X$, and the MC estimate that also samples observations from $\bar{D}_X$.

Specifically, for the empirical marginal quotient game value $h_j[M,\hat{v}^{\ME,\cP}(\cdot;x^{(r)},\bar{D}_{X},f)]$, for some $j\in M$, the estimated MISE between $h_j$ and the MC estimate $\hat{h}_j$ is given by
\begin{equation}\label{MISE}
    \widehat{\MISE}(h_j,\hat{h}_j) := \frac{1}{|\bar{D}_X|}\sum_{x\in \bar{D}_X} \big(h_j[M,\hat{v}^{\ME,\cP}(\cdot;x,\bar{D}_X,f)]-\hat{h}_j[M,\hat{v}^{\ME,\cP}(\cdot;x,\bar{D}_X,f)]\big)^2.
\end{equation}

We obtain an estimate of MISE for every MC estimate $\hat{h}_j$ produced. To construct confidence intervals for MISE, we evaluate $50$ estimates for it by running the MC approach 50 times for each individual observation in $\bar{D}_X$.

Finally, another quantity we estimate is the Relative MISE or RMISE, which is given by
\begin{equation}\label{RMISE}
    \widehat{\text{RMISE}}(h_j,\hat{h}_j) := \frac{\widehat{\text{MISE}}(h_j,\hat{h}_j)}{\frac{1}{|\bar{D}_X|}\sum_{x\in \bar{D}_X} \big(h_j[M,\hat{v}^{\ME,\cP}(\cdot;x,\bar{D}_X,f)]\big)^2}.
\end{equation}

Estimating RMISE will provide insight in the experiments where the number of predictors increases. As we add more and more predictors to a function whose output is bounded (a probability), we expect that the contributions themselves will be smaller across the predictors. This is a consequence of the efficiency property. Therefore, even though MISE will be expected to decrease without any impact from the increasing number of predictors, RMISE may be impacted.

Confidence intervals for RMISE are also constructed. As a final note before we discuss each example separately, the plots for MISE and RMISE are plotted on log-log scales to clearly show the error rate as the number of MC iterations increases.

\medskip
\noindent\textbf{Experiment 1a: Asymptotic behavior of the MC quotient Shapley estimate.}\\
The data generating model is a logistic function with four predictors.
\begin{align*}
    Y &= f(X) = \frac{\sqrt{6}}{1+\exp[-3(X_1-5)+0.2(X_2-15)-2(X_3-2/7)-5X_4]},\\
    X_1 &\sim Normal(5,1),\ X_2 \sim Gamma(3,|X_1|),\ X_3 \sim Beta(2,5),\ X_4 \sim Uniform(-1,1).
\end{align*}
Since $X_1$, $X_2$ are dependent, we form the partition $\cP$ with three groups, $\cP=\{S_1, S_2, S_3\}$ where $S_1=\{1,2\}$, $S_2=\{3\}$, and $S_3=\{4\}$. The group of interest is $S_1$, whose empirical marginal quotient Shapley $\varphi_1[M,\hat{v}^{\ME,\cP}]$ is evaluated directly.

The error estimates for MISE and RMISE for the corresponding estimate $\hat{\varphi}_1[M,\hat{v}^{\ME,\cP}]$ are given in Figure \ref{fig::qshap_asym}, showing their empirical mean after the 50 runs and corresponding $95\%$ confidence intervals. Also plotted in the figure is the theoretical rate and the estimated variance over the number of MC iterations.

\begin{figure}[H]
    \centering
       \begin{subfigure}[t]{0.45\textwidth}
           \centering
           \includegraphics[width=\textwidth]{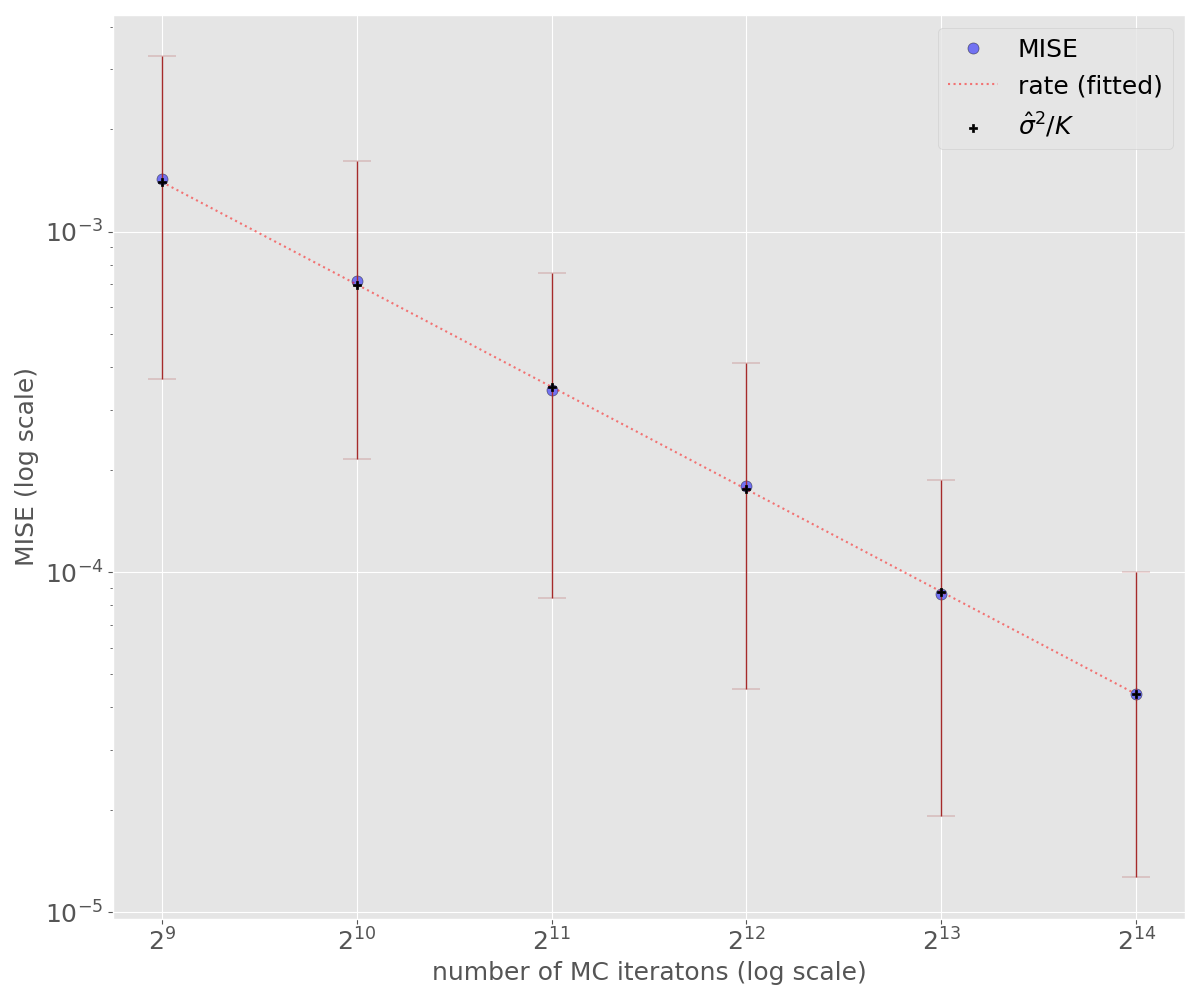}
           \caption{MISE for $\hat{\varphi}_1[M,\hat{v}^{\ME,\cP}]$}
       \end{subfigure}
       ~~
       \begin{subfigure}[t]{0.45\textwidth}
           \centering
           \includegraphics[width=\textwidth]{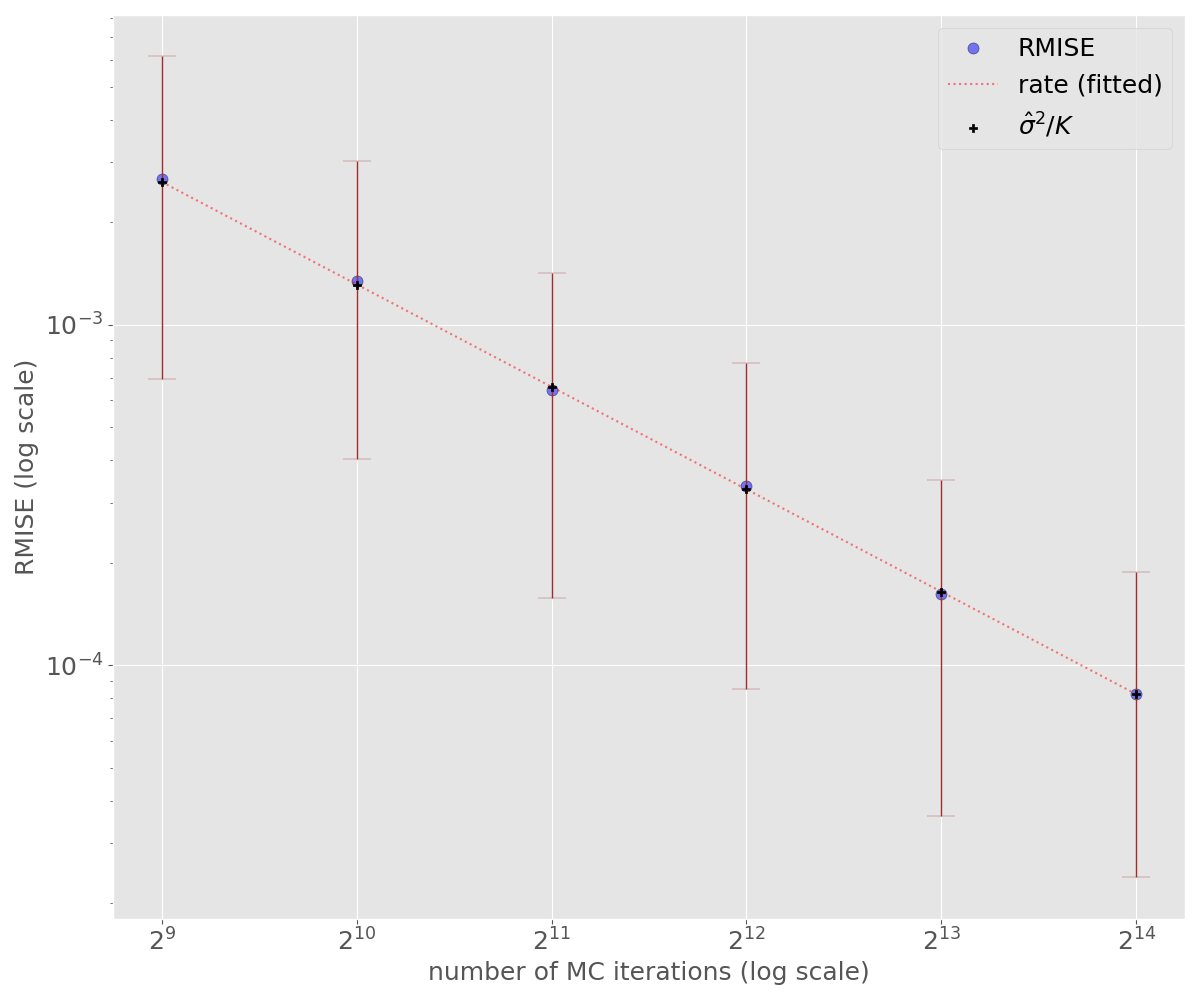}
           \caption{RMISE for $\hat{\varphi}_1[M,\hat{v}^{\ME,\cP}]$}
       \end{subfigure}
       \caption{ MC error estimate of the empirical marginal quotient Shapley for $S_1$. }\label{fig::qshap_asym}
   \end{figure}

\medskip
\noindent\textbf{Experiment 1b: Asymptotics of the MC Shapley estimate for increasing number of predictors.}\\
The data generating model is a logistic function given by
\begin{align*}
    Y &= f(X) = \frac{\sqrt{6}}{1+\exp[-3(X_1-5)+0.2(X_2-15)-2(X_3-2/7)-5X_4+\sum_{l=5}^p X_l]},\\
    X_1 &\sim Normal(5,1),\ X_2 \sim Gamma(3,|X_1|),\ X_3 \sim Beta(2,5),\ X_4 \sim Uniform(-1,1),\\
    X_l &\sim Normal(0,3), \ l\in \{5,\dots,p\},
\end{align*}
where $p$ takes values in $\{4,5,10,16\}$ (for $p=4$ we assume $\sum_{l=5}^p X_l=0$). Thus, four different models are embedded in the above formula. The purpose of this experiment is to show that as we increase the number of predictors the effect on the error is minimal.

The predictor of interest is $X_1$, whose empirical marginal Shapley value $\varphi_1[N,\hat{v}^{\ME}]$ is evaluated directly. We then perform the 50 MC runs to obtain its corresponding estimate $\hat{\varphi}_1[N,\hat{v}^{\ME}]$ and build the error plots for MISE and RMISE, shown in Figure \ref{fig::shap_asym_inc}, showing their empirical mean and corresponding $95\%$ confidence intervals for each of the four models. Also plotted in the figure is the theoretical rate and the estimated variance over the number of MC iterations, again for each of the four models.

\begin{figure}[H]
    \centering
       \begin{subfigure}[t]{0.45\textwidth}
           \centering
           \includegraphics[width=\textwidth]{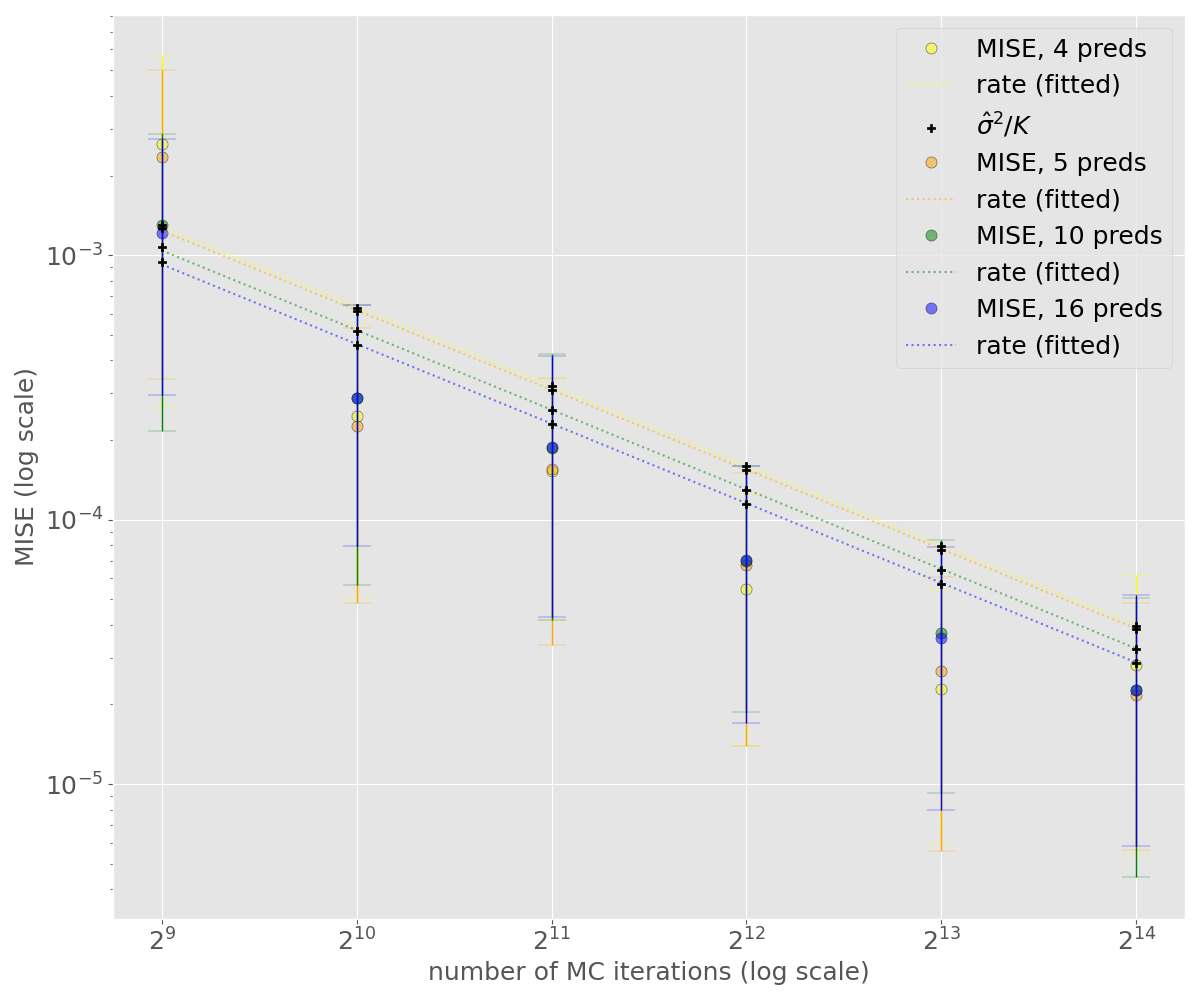}
           \caption{MISE for $\hat{\varphi}_1[N,\hat{v}^{\ME}]$}
       \end{subfigure}
       ~~
       \begin{subfigure}[t]{0.45\textwidth}
           \centering
           \includegraphics[width=\textwidth]{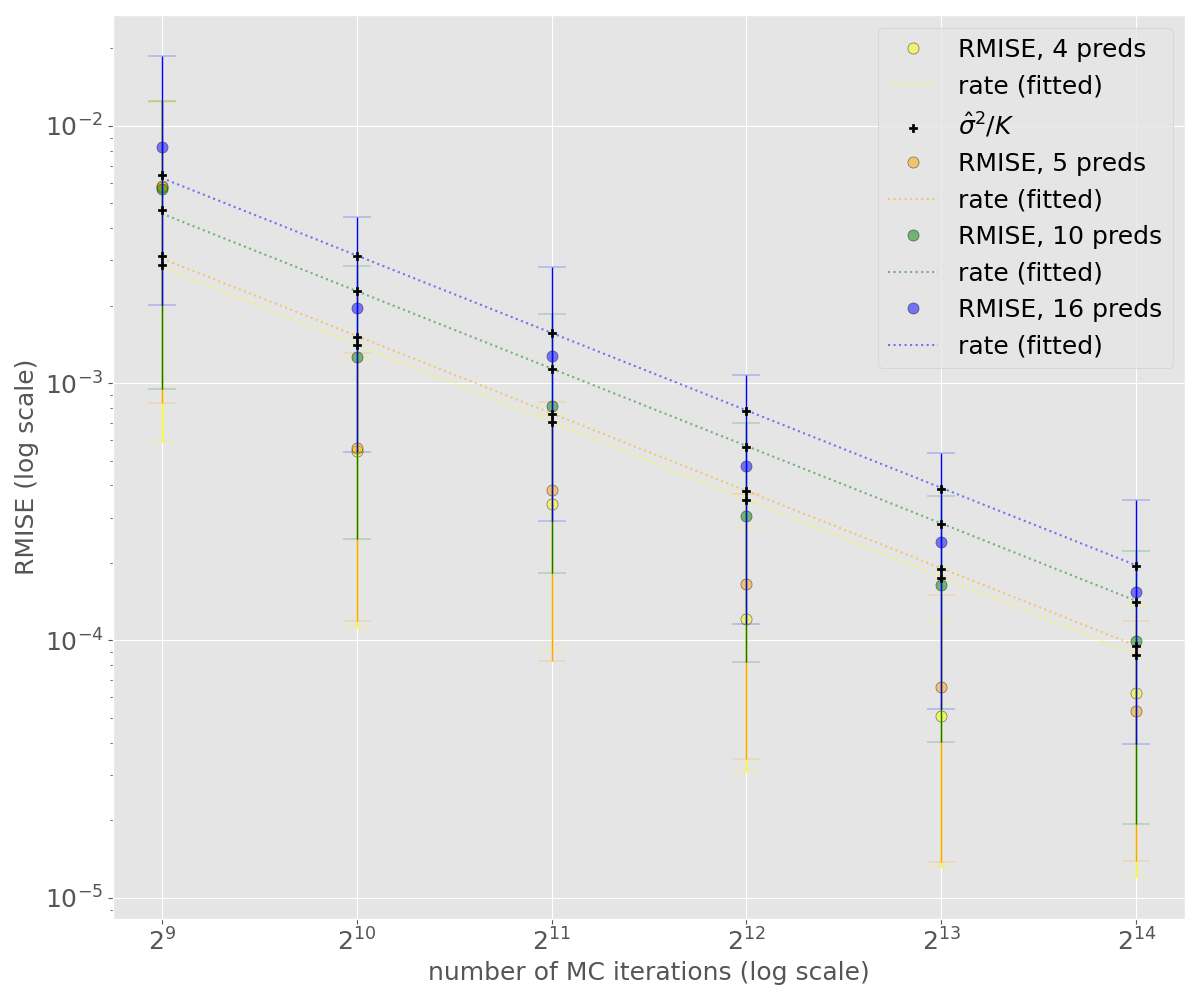}
           \caption{RMISE for $\hat{\varphi}_1[N,\hat{v}^{\ME}]$}
       \end{subfigure}
       \caption{ MC error estimate of the empirical marginal Shapley for increasing number of predictors. }\label{fig::shap_asym_inc}
   \end{figure}

\medskip
\noindent\textbf{Experiment 2a: Asymptotic behavior of the MC Owen value estimate.}\\
The data generating model is a logistic function with six predictors.
\begin{align*}
    Y &= f(X) = \frac{\sqrt{6}}{1+\exp[-3(X_1-5)+0.2(X_2-15)-2(X_3-2/7)-5X_4+X_5-0.5(\pi-\frac{1}{\pi})-X_6]},\\
    X_1 &\sim Normal(5,1),\ X_2 \sim Gamma(3,|X_1|),\ X_3 \sim Beta(2,5),\ X_4 \sim Uniform(-1,1),\\
    X_5 &= \exp(X_4)+\epsilon_1, \ \text{and} \ X_6 = X_4^2\sin(\pi X_4) + \epsilon_2, \ \epsilon_1 \sim Normal(0,0.1^2), \ \epsilon_2 \sim Normal(0,0.05^2).
\end{align*}
We form the partition $\cP$ with three groups based on dependencies, $\cP=\{S_1, S_2, S_3\}$ where $S_1=\{1,2\}$, $S_2=\{3\}$, and $S_3=\{4,5,6\}$. The predictor of interest is $X_4$, whose empirical marginal Owen value $Ow_4[N,\hat{v}^{\ME},\cP]$ is evaluated directly.

The error estimates for MISE and RMISE for the corresponding estimate $\hat{Ow}_4[N,\hat{v}^{\ME},\cP]$ are given in Figure \ref{fig::owen_asym}, showing their empirical mean after the 50 runs and corresponding $95\%$ confidence intervals. Also plotted in the figure is the theoretical rate and the estimated variance over the number of MC iterations.

\begin{figure}[H]
    \centering
       \begin{subfigure}[t]{0.45\textwidth}
           \centering
           \includegraphics[width=\textwidth]{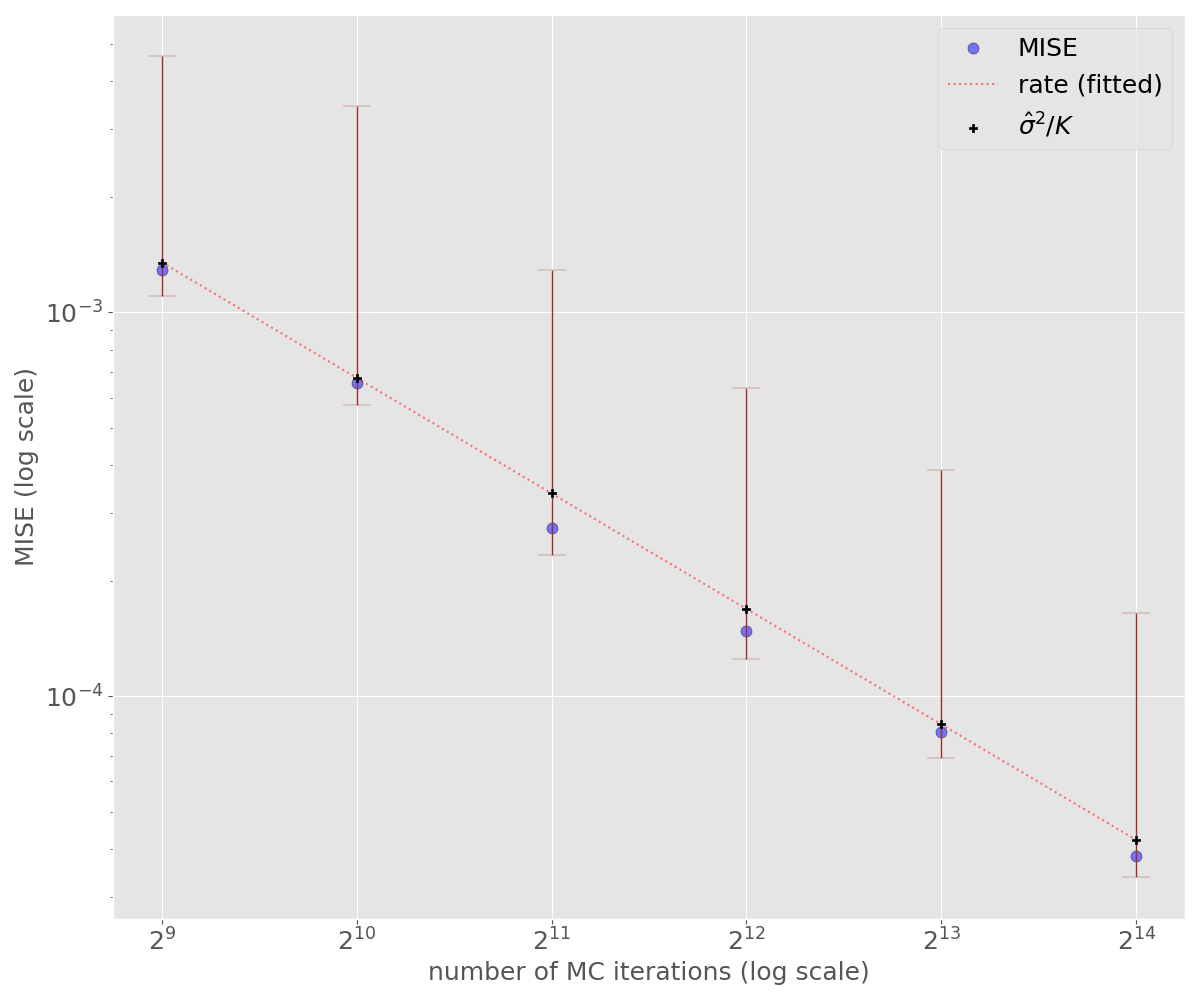}
           \caption{MISE for $\hat{Ow}_4[N,\hat{v}^{\ME},\cP]$}
       \end{subfigure}
       ~~
       \begin{subfigure}[t]{0.45\textwidth}
           \centering
           \includegraphics[width=\textwidth]{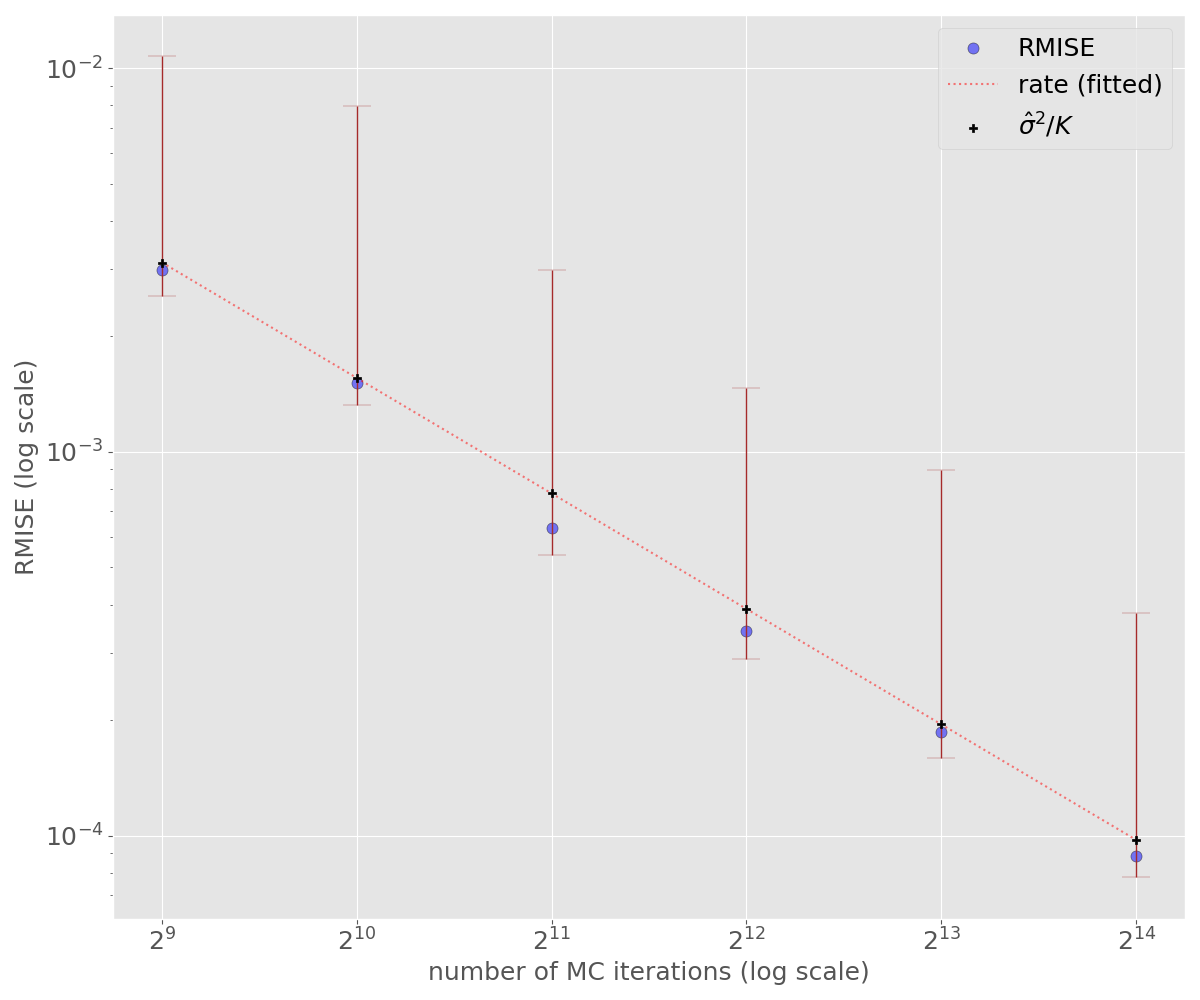}
           \caption{RMISE for $\hat{Ow}_4[N,\hat{v}^{\ME},\cP]$}
       \end{subfigure}
       \caption{ MC error estimate of the empirical marginal Owen value for $X_4$. }\label{fig::owen_asym}
   \end{figure}

\medskip
\noindent\textbf{Experiment 2b: Asymptotics of the MC Owen estimate for increasing number of predictors.}\\
The data generating model is a logistic function given by
\begin{align*}
    Y &= f(X) = \frac{\sqrt{6}}{1+\exp[-3(X_1-5)+0.2(X_2-15)-2(X_3-2/7)-5X_4+\sum_{l=5}^p X_l]},\\
    X_1 &\sim Normal(5,1),\ X_2 \sim Gamma(3,|X_1|),\ X_3 \sim Beta(2,5),\ X_4 \sim Uniform(-1,1),
\end{align*}
and $(X_5,\dots,X_p)$, where $p$ takes values in $\{6,10,14,18\}$, follows a Multivariate Normal distribution with $\E[X_l]=0$, $Var[X_l]=3$, and $Cov[X_l,X_q]=0.1$, $l,q\in \{5,\dots,p\}$, $l\ne q$. Thus, four different models are embedded in the above formula. As with Experiment 1b, the purpose of this experiment is to show that as we increase the number of predictors the effect on the error is minimal.

The predictor of interest is $X_5$, whose empirical marginal Owen value $Ow_5[N,\hat{v}^{\ME},\cP]$ is evaluated directly. The partition $\cP$ is formed by dependencies, so that $\cP=\{S_1,S_2,S_3,S_4\}$ where $S_1=\{1,2\}$, $S_2=\{3\}$, $S_3=\{4\}$, and $S_4=\{5,\dots,p\}$. We then perform the 50 MC runs to obtain the empirical marginal Owen value MC estimate $\hat{Ow}_5[N,\hat{v}^{\ME},\cP]$ and build the error plots for MISE and RMISE, shown in Figure \ref{fig::owen_asym_inc}, showing their empirical mean and corresponding $95\%$ confidence intervals for each of the four models. Also plotted in the figure is the theoretical rate and the estimated variance over the number of MC iterations, again for each of the four models.

\begin{figure}[H]
    \centering
       \begin{subfigure}[t]{0.45\textwidth}
           \centering
           \includegraphics[width=\textwidth]{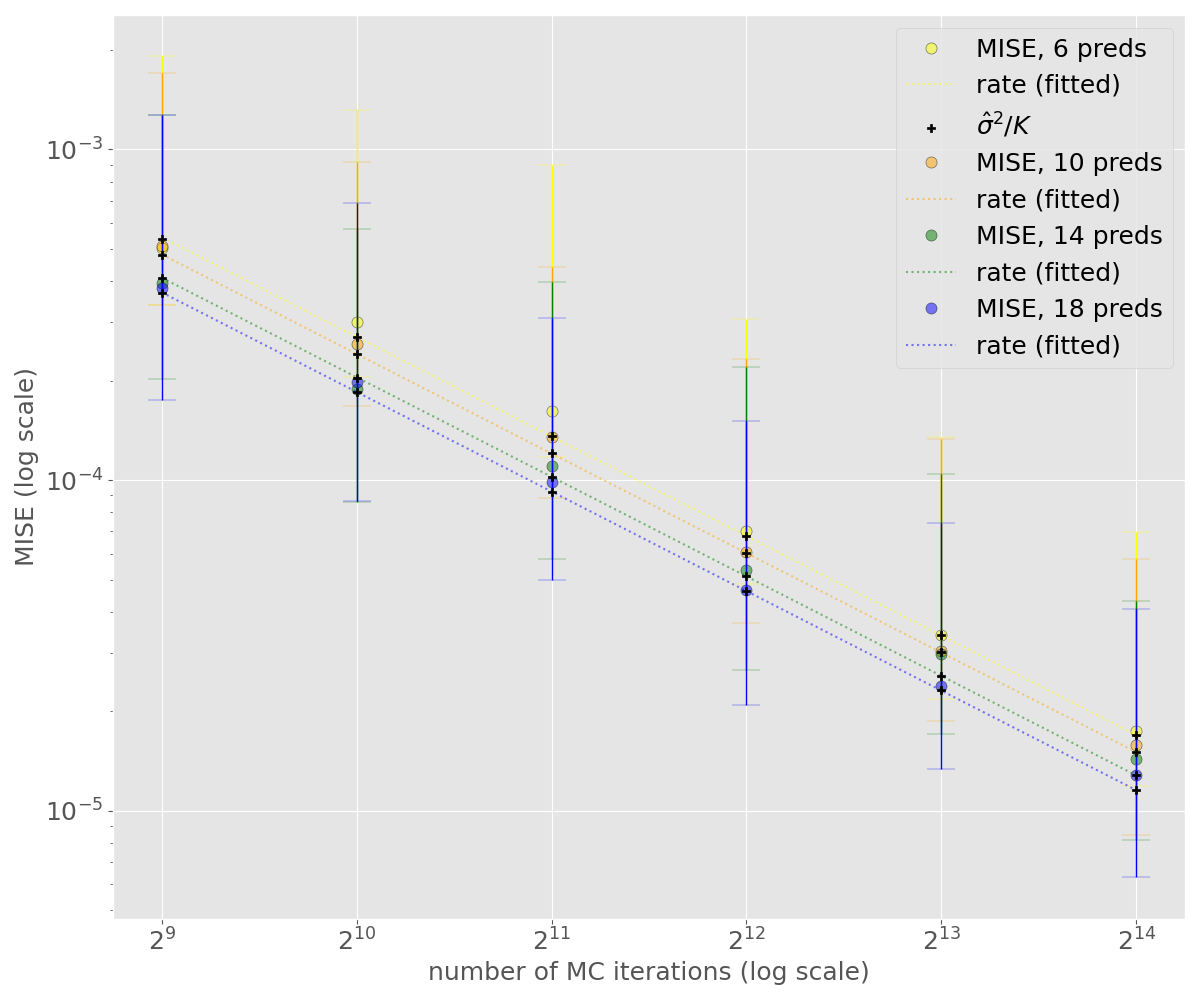}
           \caption{MISE for $\hat{Ow}_5[N,\hat{v}^{\ME},\cP]$}
       \end{subfigure}
       ~~
       \begin{subfigure}[t]{0.45\textwidth}
           \centering
           \includegraphics[width=\textwidth]{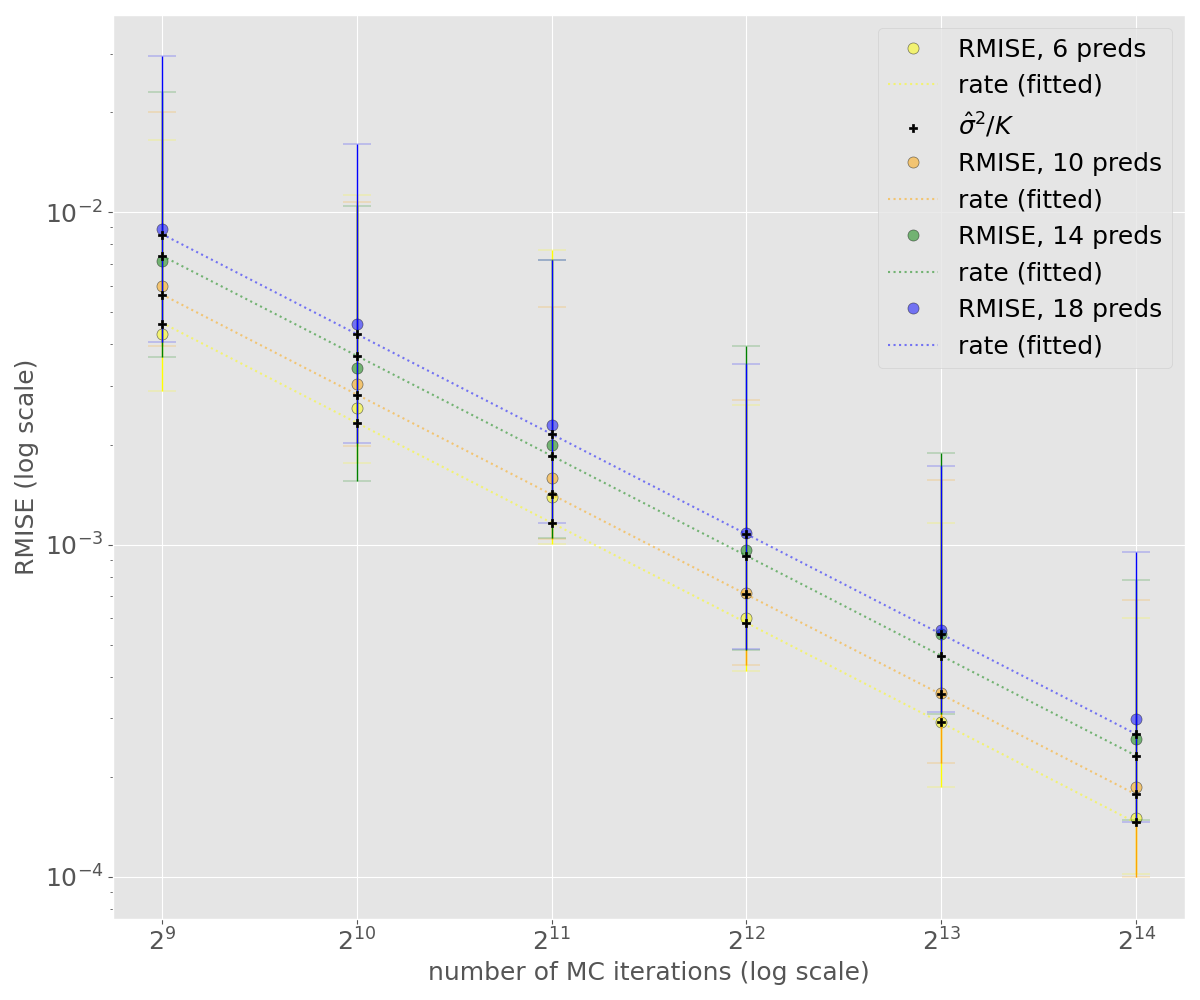}
           \caption{RMISE for $\hat{Ow}_5[N,\hat{v}^{\ME},\cP]$}
       \end{subfigure}
       \caption{ MC error estimate of the empirical marginal Owen value as the number of predictors increases. }\label{fig::owen_asym_inc}
   \end{figure}

\medskip
\noindent\textbf{Experiment 3a: Asymptotic behavior of the MC two-step Shapley estimate.}\\
The setup is the same as with Experiment 2a. The difference is that for the predictor of interest $X_4$ we estimate the empirical marginal two-step Shapley value $TSh_4[N,\hat{v}^{\ME},\cP]$.

The error estimates for MISE and RMISE for $\hat{TSh}_4[N,\hat{v}^{\ME},\cP]$ are given in Figure \ref{fig::ts_asym}, showing their empirical mean after the 50 runs and corresponding $95\%$ confidence intervals. Also plotted in the figure is the theoretical rate and the estimated variance over the number of MC iterations.

\begin{figure}[H]
    \centering
       \begin{subfigure}[t]{0.45\textwidth}
           \centering
           \includegraphics[width=\textwidth]{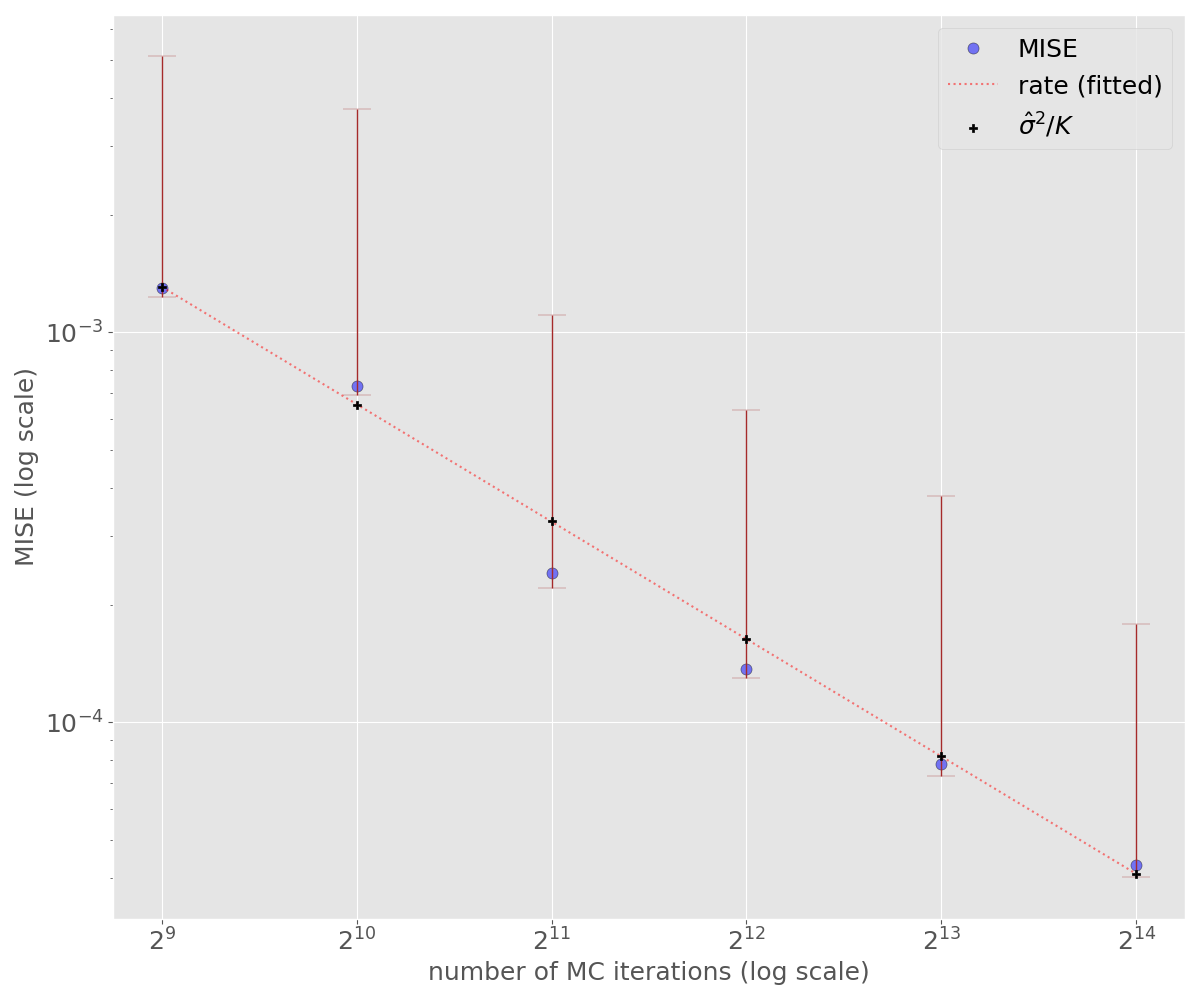}
           \caption{MISE for $\hat{TSh}_4[N,\hat{v}^{\ME},\cP]$}
       \end{subfigure}
       ~~
       \begin{subfigure}[t]{0.45\textwidth}
           \centering
           \includegraphics[width=\textwidth]{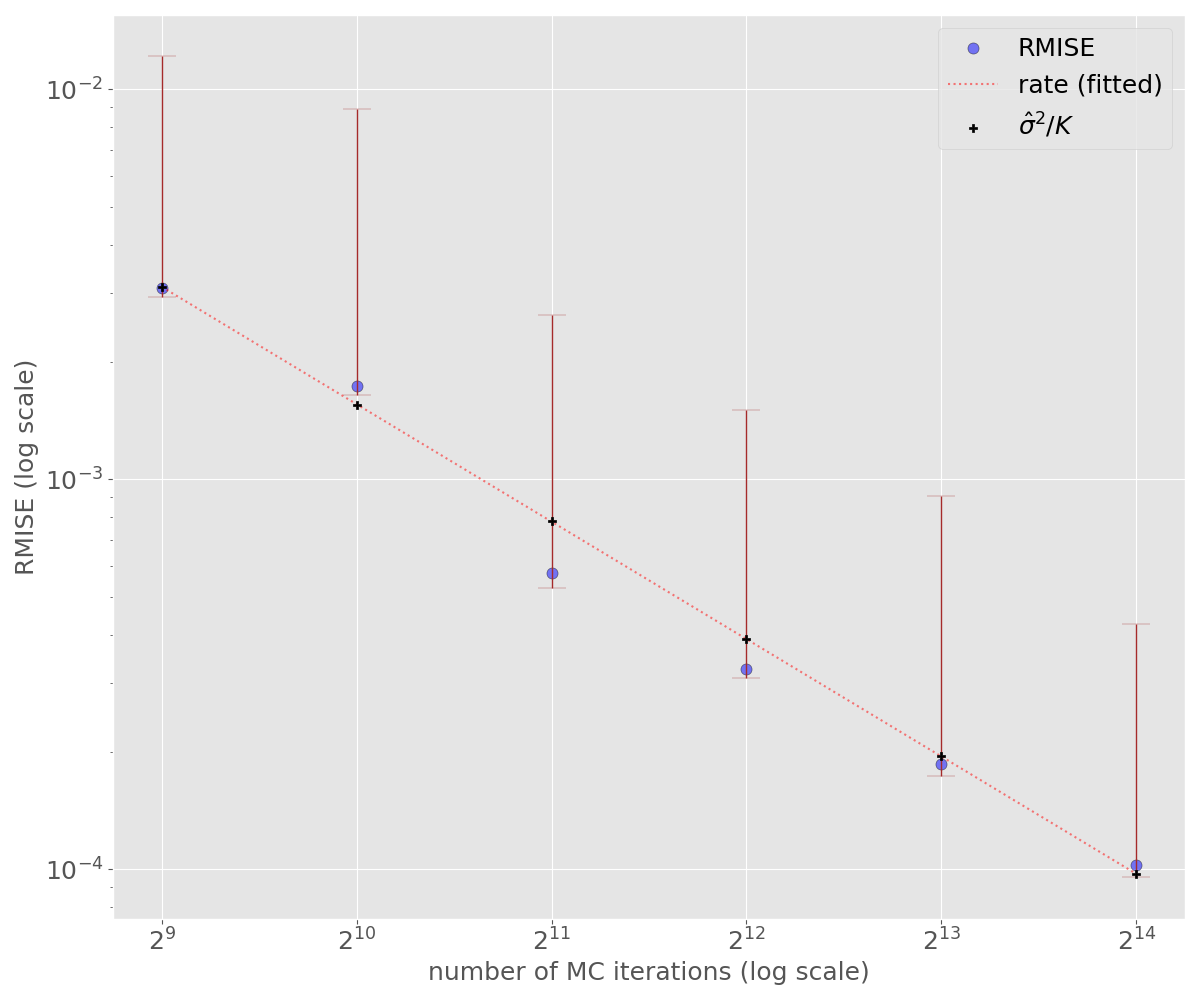}
           \caption{RMISE for $\hat{TSh}_4[N,\hat{v}^{\ME},\cP]$}
       \end{subfigure}
       \caption{ MC error estimate of the empirical marginal two-step Shapley for $X_4$. }\label{fig::ts_asym}
   \end{figure}

\medskip
\noindent\textbf{Experiment 3b: Asymptotics for MC two-step Shapley as the number of predictors increases.}\\
The setup is the same as with Experiment 2b. The difference is that for the predictor of interest $X_5$ we estimate the empirical marginal two-step Shapley value $TSh_5[N,\hat{v}^{\ME},\cP]$.

We then perform the 50 MC runs to obtain the empirical marginal two-step Shapley value MC estimate $\hat{TSh}_5[N,\hat{v}^{\ME},\cP]$ and build the error plots for MISE and RMISE, shown in Figure \ref{fig::ts_asym_inc}, showing their empirical mean and corresponding $95\%$ confidence intervals for each of the four models. Also plotted in the figure is the theoretical rate and the estimated variance over the number of MC iterations, again for each of the four models.

\begin{figure}[H]
    \centering
       \begin{subfigure}[t]{0.45\textwidth}
           \centering
           \includegraphics[width=\textwidth]{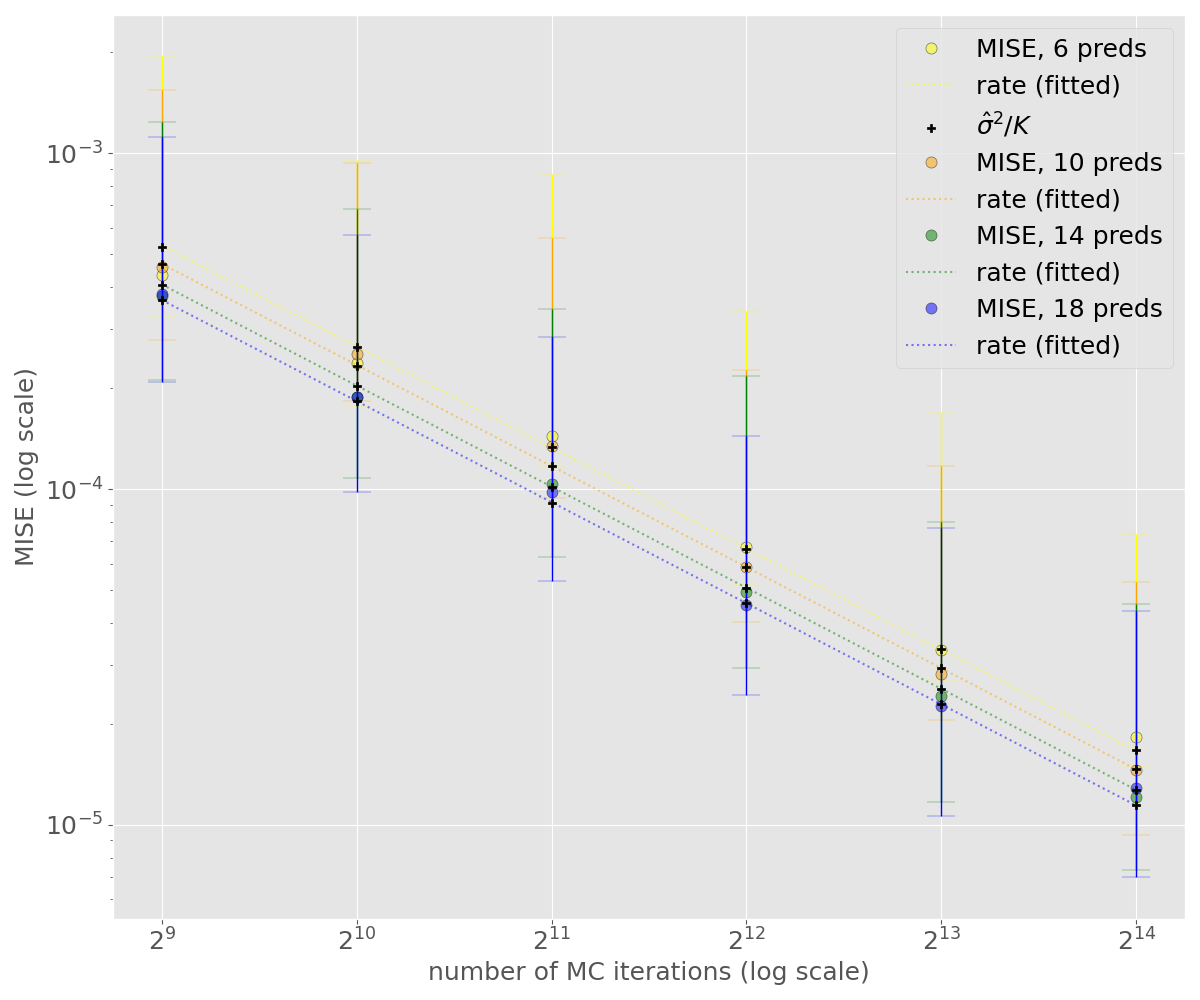}
           \caption{MISE for $\hat{TSh}_5[N,\hat{v}^{\ME},\cP]$}
       \end{subfigure}
       ~~
       \begin{subfigure}[t]{0.45\textwidth}
           \centering
           \includegraphics[width=\textwidth]{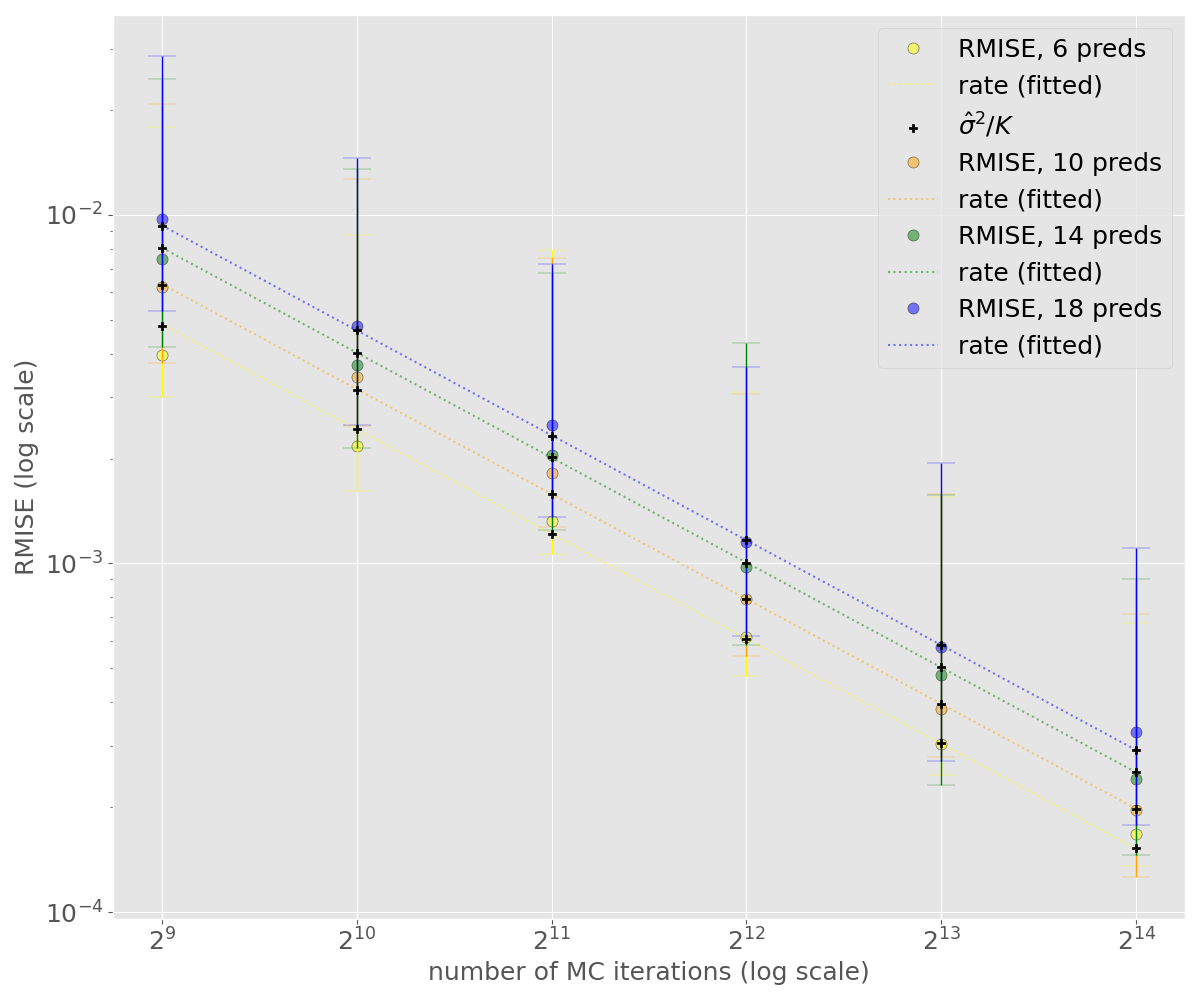}
           \caption{RMISE for $\hat{TSh}_5[N,\hat{v}^{\ME},\cP]$}
       \end{subfigure}
       \caption{ MC error estimate of the empirical marginal two-step Shapley as the number of predictors increases. }\label{fig::ts_asym_inc}
   \end{figure}

\medskip
\noindent\textit{Conclusion.} As we can see from the plots, the error decays as the theory dictates in all experiments. Furthermore, we see that the increase in number of predictors has minimal impact on the absolute error. However, as discussed earlier, for the relative MISE we see that the error does slightly increase for more predictors due to the contributions being smaller as they are dispersed across more predictors.

\begin{appendices}

\section*{Appendix}
\section{Monte Carlo integration}
The foundation of Monte Carlo approximation is the ability to draw independent samples from a distribution given by some measure $P$. Details on Monte Carlo theory can be found in \cite{MCbook} and \cite[\S 24]{Wasserman}. Assuming that $F(Z)$ has bounded variance, where $Z\sim P$, the weak law of large numbers guarantees that the estimator
\[
    I_K = \frac{1}{K}\sum_{k=1}^K F\left(Z^{(k)}\right)
\]
is a consistent estimator of $\E[F(Z)]=\int F(Z)dP$, where $\{Z^{(k)}\}$ is an i.i.d. sequence of random variables distributed according to $P$. The estimation of the error of $I_K$ is given by $\sqrt{Var[I_K]}=\sqrt{Var[F(Z)]}\cdot K^{-1/2}$. Furthermore, $I_K$ is unbiased since $\E[I_K] = \E[F(Z)]$; see Ch. I, $\S 5$ of \cite{Shiryaev}.

\begin{remark}\rm
    The weak law of large numbers does not require $Var[F(Z)]< \infty$ for convergence to occur. It will simply be the case that the convergence may be slower than $O(K^{-1/2})$ by only having the integrability of $F(Z)$.
\end{remark}

\end{appendices}

\end{document}